%% file: main.tex
\documentclass[article]{article}
\usepackage{uai2018}
\usepackage{times}
\usepackage{hyperref}
\usepackage{color}
\usepackage{footnote}
\usepackage{url}
\usepackage{natbib}
\usepackage{booktabs}

\usepackage{placeins}
\usepackage{amsmath}
\usepackage{amsthm}
\usepackage{amssymb}
\usepackage{caption}
\usepackage{subcaption}
\usepackage{tikz}
\usepackage{comment}
\usepackage{bm}
\usepackage{array}
\usepackage{mathtools}
\usepackage{algorithm}
\usepackage{algorithmic}

\newtheorem{definition}{Definition}

\newtheorem{theorem}{Theorem}
\newtheorem{lemma}{Lemma}

\newcommand\numberthis{\refstepcounter{equation}\tag{\theequation}}

\newcommand{\bv}{\boldsymbol{b}}
\newcommand{\cv}{\boldsymbol{c}}

\newcommand{\pv}{\boldsymbol{p}}
\newcommand{\qv}{\boldsymbol{q}}
\newcommand{\nv}{\boldsymbol{n}}

\newcommand{\rv}{\boldsymbol{r}}
\newcommand{\sv}{\boldsymbol{s}}
\newcommand{\tv}{\boldsymbol{t}}
\newcommand{\uv}{\boldsymbol{u}}
\newcommand{\vv}{\boldsymbol{v}}

\newcommand{\xv}{\boldsymbol{x}}

\newcommand{\zv}{\boldsymbol{z}}

\newcommand{\Ev}[0]{{{\bf E}}}

\newcommand{\Dbc}[0]{{\mathcal{D}}}
\newcommand{\zerov}[0]{{{\bf 0}}}

\newcommand{\pt}{p_{\theta}}
\newcommand{\qt}{q_{\theta}}

\newcommand{\atoms}{\mathrm{atoms}}
\def\E{\mathbb{E}}
\newcommand{\Eb}{\mathbb{E}}
\newcommand{\gammav}{{\boldsymbol{\gamma}}}
\newcommand{\lambdav}{{\boldsymbol{\lambda}}}
\newcommand{\epsilonv}{{\boldsymbol{\epsilon}}}

\newcommand{\KL}{D_{\mathrm{KL}}}
\newcommand{\JS}{D_{\mathrm{JS}}}
\newcommand{\MMD}{D_{\mathrm{MMD}}}

\newcommand{\ELBO}{\mathcal{L}_{\mathrm{ELBO}}}

\newcommand{\D}{\mathcal{D}}

\newcommand{\bb}[1]{\mathbb{#1}}

\newcommand{\mc}[1]{\mathcal{#1}}

\newcommand{\rmodel}{{r_{\mathrm{model}}}}
\newcommand{\defeq}{\stackrel{\mathclap{\tiny\mbox{def}}}{=}}

\makesavenoteenv{tabular}
\makesavenoteenv{table}

\newcommand\blfootnote[1]{%
  \begingroup
  \renewcommand\thefootnote{}\footnote{#1}%
  \addtocounter{footnote}{-1}%
  \endgroup
}

\title{The Information Autoencoding Family: \\ A Lagrangian Perspective on Latent Variable Generative Models}
\author{ {\bf Shengjia Zhao} \\
Computer Science Department \\
Stanford University \\
sjzhao@stanford.edu \\
\And
{\bf Jiaming Song}  \\
Computer Science Department \\
Stanford University \\
tsong@stanford.edu \\
\And
{\bf Stefano Ermon}   \\
Computer Science Department \\
Stanford University \\
ermon@stanford.edu \\
}
\begin{document}

\maketitle

\begin{abstract}
A large number of objectives have been proposed to train latent variable generative models. We show that many of them are Lagrangian dual functions of the same primal optimization problem. The primal problem optimizes the mutual information between latent and visible variables, subject to the constraints of accurately modeling the data distribution and performing correct amortized inference. By choosing to maximize or minimize mutual information, and choosing different Lagrange multipliers, we obtain different objectives including InfoGAN, ALI/BiGAN, ALICE, CycleGAN, beta-VAE, adversarial autoencoders, AVB, AS-VAE and InfoVAE. Based on this observation, we provide an exhaustive characterization of the statistical and computational trade-offs made by all the training objectives in this class of Lagrangian duals. Next, we propose a dual optimization method where we optimize model parameters as well as the Lagrange multipliers. This method achieves Pareto optimal solutions in terms of optimizing information and satisfying the constraints.\blfootnote{To appear in Conference on Uncertainty in Artificial Intelligence (UAI 2018), Monterey, California, USA}
\end{abstract}

\input{1_introduction}

\input{2_background}

\input{3_primal}

\input{4_tractable}
\input{5_lagvae}

\input{5_evaluation}

\input{8_conclusion}

\subsubsection*{Acknolwledgements}
This research was supported by Intel Corporation, TRI, NSF (\#1651565, \#1522054, \#1733686) and FLI (\#2017-158687).

\bibliographystyle{icml2018}
\bibliography{ref}

\newpage

\input{9_appendix}

\end{document}

%% file: 1_introduction.tex
\section{INTRODUCTION}

Latent variable generative models
are designed to accomplish a wide variety of tasks in computer vision~\citep{deconvolutional_gan2015,kuleshov2017deep}, natural language processing~\citep{improved_vae_nlp2017}, reinforcement learning~\citep{infogail2017}, compressed sensing~\cite{dhar2018sparse},etc. 
Prominent examples include 
Variational Autoencoders (VAE, \citet{autoencoding_variational_bayes2013,variational_dbn_stochastic_bp2014}), with extensions such as $\beta$-VAE \citep{beta_vae2016}, Adversarial Autoencoders \citep{adversarial_autoencoder2015}, and InfoVAE \citep{zhao2017infovae}; 
Generative Adversarial Networks \citep{generative_adversarial_nets2014}, with extensions such as ALI/BiGAN \citep{adversarially_learned_inference2016,bigan2016}, InfoGAN \citep{infogan2016} and ALICE~\citep{alice2017}; hybrid objectives such as CycleGAN \citep{cyclegan2017}, DiscoGAN \citep{discogan2017}, AVB \citep{adversarial_variational_bayes2017} and AS-VAE \citep{pu2017adversarial}.
All these models attempt to fit an empirical data distribution, but differ in multiple ways: how they measure the similarity between distributions; whether or not they allow for efficient (amortized) inference; whether the latent variables should retain or discard information about the data;  and how the model is optimized, which can be likelihood-based or likelihood-free~\citep{mohamed2016learning,grover2018flow}. 


In this paper, we generalize existing training objectives for latent variable generative models. We show that all the above training objectives can be viewed as  \emph{Lagrangian dual functions} 
of a constrained optimization problem (primal problem). 
The primal problem optimizes over the parameters of a generative model and an (amortized) inference distribution. The optimization objective is to maximize or minimize mutual information between latent and observed variables;
the constraints (which we term ``\emph{consistency constraints}'') are to accurately model the data distribution and to perform correct amortized inference. 
By considering the Lagrangian dual function and different settings of the Lagrange multipliers, we can obtain all the aforementioned generative modeling training objectives. 
Surprisingly, under mild assumptions, the aforementioned objectives can be linearly combined to produce every possible primal objective/multipliers in this model family. 

In Lagrangian dual optimization, the dual function is maximized with respect to the Lagrange multipliers, and minimized with respect to the primal parameters. Under strong duality, the optimal parameters found by this procedure also solve the original primal problem. 
However, the aforementioned objectives use fixed (rather than maximized) multipliers. As a consequence, strong duality does not generally hold. 

To overcome this problem, we propose a new learning approach where the Lagrange multipliers are also optimized. We show that strong duality holds in distribution space, so this optimization procedure is guaranteed to optimize the primal objective while satisfying the consistency constraints. 
As an application of this approach, we propose \emph{Lagrangian VAE}, a Lagrangian optimization algorithm for the InfoVAE~\citep{zhao2017infovae} objective. Lagrangian VAE can explicitly trade-off optimization of the primal objective and consistency constraint satisfaction. 
In addition, both theoretical properties (of Lagrangian optimization) and empirical experiments show that solutions obtained by Lagrangian VAE \emph{Pareto dominate} solutions obtained with InfoVAE: Lagrangian VAE either obtains better mutual information or better constraint satisfaction, regardless of the hyper-parameters used by either method. 

%% file: 2_background.tex
\section{BACKGROUND}
\label{sec:background}

We consider two groups of variables: observed variables $\xv \in \mc{X}$ and latent variables $\zv \in \mc{Z}$. Our algorithm receives input distributions $q(\xv), p(\zv)$ over $\xv$ and $\zv$ respectively.
Each distribution is either specified \textit{explicitly} through a tractable analytical expression such as $\mathcal{N}(0, I)$, or \textit{implicitly} through a set of samples. 
For example, in latent variable generative modeling of images~\citep{autoencoding_variational_bayes2013,generative_adversarial_nets2014}, $\mc{X}$ is the space of images, and $\mc{Z}$ is the space of latent features. $q(\xv)$ is a dataset of sample images, and $p(\zv)$ is a simple ``prior'' distribution, e.g., a Gaussian; in unsupervised image translation~\citep{cyclegan2017}, $\mc{X}$ and $\mc{Z}$ are both image spaces and $q(\xv)$, $p(\zv)$ are sample images from two different domains (e.g., pictures of horses and zebras). 

The underlying joint distribution on $(\xv, \zv)$ is not known, and
we are not given any sample from it.
Our goal is to nonetheless learn some model of the joint distribution $\rmodel(\xv, \zv)$ with the following desiderata: 

\textbf{Desideratum 1. Matching Marginal} The marginals of $\rmodel(\xv, \zv)$ over $\xv, \zv$ respectively match the provided distributions $q(\xv), p(\zv)$.

\textbf{Desideratum 2. Meaningful Relationship} $\rmodel(\xv, \zv)$ captures a meaningful relationship between $\xv$ and $\zv$.
For example, in latent variable modeling of images, the latent variables $\zv$ should correspond to semantically meaningful features describing the image $\xv$. In unsupervised image translation, $r_{\mathrm{model}}(\xv, \zv)$ should capture the ``correct'' pairing between $\xv$ and $\zv$.  


We address desideratum 1 in this section, and desideratum 2 in section 3. The joint distribution $\rmodel(\xv, \zv)$ can be represented in factorized form by chain rule.
To do so, we define conditional distribution families $\lbrace p_{\theta^p}(\xv|\zv), \theta^p \in \Theta^p \rbrace$ and $\lbrace q_{\theta^q}(\zv|\xv), \theta^q \in \Theta^q \rbrace$. We require that for any $\zv$ we can both efficiently sample from $p_{\theta^p}(\xv|\zv)$ and compute $\log p_{\theta^p}(\xv|\zv)$, and similarly for $q_{\theta^q}(\zv|\xv)$. For compactness we use $\theta=(\theta^p,\theta^q)$ to denote the parameters of both distributions $p_\theta$ and $q_\theta$. We define the joint distribution $\rmodel(\xv, \zv)$ in two ways: 
\begin{align*}
\rmodel(\xv, \zv) \defeq p_\theta(\xv, \zv) \defeq p(\zv) p_\theta(\xv|\zv) \numberthis \label{eq:model_form1}
\end{align*}
and symmetrically 
\begin{align*}
\rmodel(\xv, \zv) \defeq q_\theta(\xv, \zv) \defeq q(\xv) q_\theta(\zv|\xv) \numberthis \label{eq:model_form2}
\end{align*}
Defining the model in two (redundant) ways seem unusual but has significant computational advantages: 
given $\xv$ we can tractably sample $\zv$,  
and vice versa. For example, in latent variable models, given observed data $\xv$ we can sample latent features from $\zv \sim q_\theta(\zv|\xv)$ (amortized inference), and given latent feature $\zv$ we can generate novel samples from $\xv \sim p_\theta(\xv|\zv)$ (ancestral sampling). 

If the two definitions (\ref{eq:model_form1}), (\ref{eq:model_form2}) are \textbf{consistent}, which we define as $p_\theta(\xv, \zv) = q_\theta(\xv, \zv)$, we automatically satisfy desideratum 1:
\begin{align*}
\rmodel(\xv) &= \int_{\zv} \rmodel(\xv, \zv) d\zv = \int_{\zv} q_\theta(\xv, \zv) d\zv = q(\xv) \\
\rmodel(\zv) &= \int_{\xv} \rmodel(\xv, \zv) d\xv = \int_{\xv} p_\theta(\xv, \zv) d\xv = p(\zv)
\end{align*}

Based on this observation, we can design objectives that encourage consistency. Many latent variable generative models fit into this framework. For example, variational autoencoders~(VAE,~\citet{autoencoding_variational_bayes2013}) enforce consistency by minimizing the KL divergence: 
\begin{align*}
\min_\theta \KL(q_\theta(\xv, \zv) \Vert p_\theta(\xv, \zv)) 
\end{align*}
This minimization is equivalent to maximizing the evidence lower bound ($\ELBO$)~\citep{autoencoding_variational_bayes2013}:
\begin{align*}
&\ \max_\theta -\KL(q_\theta(\xv, \zv) \Vert p_\theta(\xv, \zv)) \numberthis \label{eq:vae_dual_form} \\
& = -\Eb_{q_\theta(\xv, \zv)}\left[ \log (q_\theta(\zv|\xv) q(\xv)) - \log (p_\theta(\xv|\zv)p(\zv))\right] \\
& \begin{array}{l} =\bb{E}_{q_\theta(\xv, \zv)}[\log p_\theta(\xv|\zv)] + H_q(\xv) \\
\qquad - \Eb_{q(x)}\left[ \KL(q_\theta(\zv|\xv) \Vert p(\zv)) \right] \end{array} \\
& \left. \begin{array}{l} \equiv \bb{E}_{q_\theta(\xv, \zv)}[\log p_\theta(\xv|\zv)] \\
\qquad - \Eb_{q(x)}\left[ \KL(q_\theta(\zv|\xv) \Vert p(\zv)) \right] \end{array}  \right\rbrace \ELBO \numberthis \label{eq:vae_computable_form}
\end{align*}
where $H_q(\xv)$ is the entropy of $q(\xv)$ and is a constant that can be ignored for the purposes of optimization over model parameters $\theta$ (denoted $\equiv$). 

As another example, BiGAN/ALI \citep{bigan2016,ali2016} use an adversarial discriminator to approximately minimize the Jensen-Shannon divergence 
\[ \min_{\theta} \JS(q_\theta(\xv, \zv) \Vert p_\theta(\xv, \zv))  \]


Many other ways of enforcing consistency are possible. Most generally, we can enforce consistency with a vector of divergences
$\mc{D} = [D_1, \dots, D_m]$, where each $D_i$ takes two probability measures as input, and outputs a non-negative value which is zero if and only if the two input measures are the same.
Examples of possible divergences include Maximum Mean Discrepancy~(MMD,~\citet{mmd_statistics2007}), denoted $\MMD$; 
Wasserstein distance~\citep{wgan2017}, denoted 
$D_{\mathrm{W}}$; $f$-divergences~\citep{f_gan2016}, denoted $D_f$; and Jensen-Shannon divergence~\citep{generative_adversarial_nets2014}, denoted $\JS$.


Each $D_i$ can be any divergence applied to a pair of probability measures. The pair of probability measures can be defined over either both variables  $(\xv, \zv)$, a single variable $\xv$, $\zv$, or conditional $\xv \vert \zv$, $\zv \vert \xv$. If the probability measure is defined over a conditional  $\xv \vert \zv$, $\zv \vert \xv$, we also take expectation over the conditioning variable with respect to $p_\theta$ or $q_\theta$. Some examples of $D_i$ are:
\begin{align*}
\Eb_{q_\theta(z)} [ & D_{\mathrm{KL}}(q_\theta(x|z) \Vert p_\theta(x|z)) ] \\
& \MMD(q_\theta(\zv) \Vert p(\zv)) \\
& D_{\mathrm{W}}(p_\theta(\xv, \zv)\Vert q_\theta(\xv, \zv)) \\
\Eb_{q(x)}[ & D_f(q_\theta(\zv|\xv) \Vert p_\theta(\zv|\xv)) ] \\ 
& \JS(q(\xv) \Vert p_\theta(\xv))
\end{align*}
We only require that 
\[ D_i = 0, \forall i \in \{1, \dots, m\} \iff p_\theta(\xv, \zv) = q_\theta(\xv, \zv)\]
so $\mc{D} = {\bf 0}$ implies consistency. Note that each $D_i$ implicitly depends on the parameters $\theta$ through $p_\theta$ and $q_\theta$, but notationally we neglect this for simplicity. 


Enforcing consistency $p_\theta(\xv, \zv) = q_\theta(\xv, \zv)$ by $\mc{D} = \zerov$ satisfies desideratum 1 (matching marginal), but does not directly address desideratum 2 (meaningful relationship). A large number of joint distributions can have the same marginal distributions $p(\zv)$ and $q(\xv)$ (including ones where $\zv$ and $\xv$ are independent), and only a small fraction of them encode meaningful models. 

%% file: 3_primal.tex
\section{GENERATIVE MODELING AS CONSTRAINT OPTIMIZATION}

To address desideratum 2, we 
modify the training objective and specify additional preferences among consistent $p_\theta(\xv, \zv)$ and $q_\theta(\xv, \zv)$.
Formally we solve the following primal optimization problem
\begin{gather}
\min_\theta  f(\theta) \quad \text{subject to }\mc{D} = \zerov \label{eq:primal}
\end{gather}
where $f(\theta)$ encodes our preferences over consistent distributions, and depends on $\theta$ through $p_\theta(\xv|\zv)$ and $q_\theta(\xv|\zv)$. 

An important preference is the mutual information between $\xv$ and $\zv$. Depending on the downstream application, we may maximize mutual information~\citep{lossy_vae2016,zhao2017infovae,alice2017,infogan2016} 
so that the features (latent variables) $\zv$ can capture as much  information as possible about $\xv$, or minimize mutual information~\citep{zhao2017infovae,beta_vae2016,information_bottleneck2015, shamir2010learning} to achieve compression. 
To implement mutual information preference we consider the following objective
\begin{align*}
f_I(\theta; \alpha_1, \alpha_2) = \alpha_1 I_{q_\theta}(\xv; \zv) + \alpha_2 I_{p_\theta}(\xv; \zv) \numberthis \label{eq:objective_information}
\end{align*}
where $I_{p_\theta}(\xv; \zv) = \bb{E}_{p_\theta(\xv, \zv)}[\log p_\theta(\xv, \zv) - \log \pt(\xv)p(\zv)]$ is the mutual information under $p_\theta(\xv, \zv)$, and $I_{q_\theta}(\xv; \zv)$ is their mutual information under $q_\theta(\xv, \zv)$. 

The optimization problem in Eq.(\ref{eq:primal}) with mutual information $f(\theta)$ in Eq.(\ref{eq:objective_information}) has the following Lagrangian dual function:
\begin{equation} 
\alpha_1 I_{q_\theta}(\xv; \zv) + \alpha_2 I_{p_\theta}(\xv; \zv)  + \lambdav^\top \mc{D} \label{eq:lag_dual}
\end{equation}
where $\lambdav = [\lambda_1, \dots, \lambda_m]$ is a vector of Lagrange multipliers, one for each of the $m$ consistency constraints in $\mc{D} = [D_1, \dots, D_m]$.


In the next section, we will show that many existing training objectives for generative models 
minimize the Lagrangian dual in Equation~\ref{eq:lag_dual} for some \emph{fixed} $\alpha_1, \alpha_2$, $\mc{D}$ and $\lambdav$. 
However, dual optimization requires maximization over the dual parameters $\lambdav$, which should \emph{not} be kept fixed. We discuss dual optimization in Section 5. 


\section{GENERALIZING OBJECTIVES WITH FIXED MULTIPLIERS}


\begin{savenotes}
\begin{table*}[t]
\centering
\begin{minipage}{\linewidth}
\begin{tabular}{c|c|c|c}
\hline
$f(p, q)$ & Likelihood Based & Unary Likelihood Free & Binary Likelihood Free \\ 
\hline
0 & VAE~\citep{autoencoding_variational_bayes2013} & VAE-GAN~\citep{adversarial_autoencoder2015} & ALI~\citep{ali2016} \\
$ \alpha_1 I_q $ & $\beta$-VAE~\citep{beta_vae2016} & InfoVAE~\citep{zhao2017infovae}  & ALICE~\citep{alice2017} \\ 
$ \alpha_2 I_p $ & VMI~\citep{barber2003algorithm} & InfoGAN~\citep{infogan2016} & - \\ 
$ \alpha_1 I_q + \alpha_2 I_p$ & - & CycleGAN~\citep{cyclegan2017} & AS-VAE \citep{pu2017adversarial}  \\
\hline
\end{tabular}
\end{minipage}
\caption{For each choice of $\alpha$ and computability class (Definition \ref{def:tractable_families}) we list the corresponding existing model. Several other objectives are also Lagrangian duals, but they are not listed because they are similar to models in the table. These objectives include DiscoGAN \citep{discogan2017}, BiGAN \citep{bigan2016}, AAE \citep{adversarial_autoencoder2015}, WAE \citep{tolstikhin2017wasserstein}. 
}
\label{tab:params}
\end{table*}
\end{savenotes}

Several existing objectives for latent variable generative models can be rewritten in the dual form of Equation~\ref{eq:lag_dual} with \emph{fixed} Lagrange multipliers. We provide several examples here and provide more in Appendix A.  

\paragraph{VAE~\citep{autoencoding_variational_bayes2013}}
Per our discussion in Section~\ref{sec:background}, the VAE training objective commonly written as ELBO maximization in Eq.(\ref{eq:vae_computable_form}) is actually equivalent to Equation~\ref{eq:vae_dual_form}. This is a dual form where we set $\Dbc = [ \KL(q_\theta(\xv, \zv)\Vert p_\theta(\xv, \zv) ]$, $\alpha_1 = \alpha_2 = 0$ and $\lambdav = 1$. 
Because $\alpha_1 = \alpha_2 = 0$, this objective has no information preference, confirming previous observations that the learned distribution can have high, low or zero mutual information between $x$ and $z$.~\cite{lossy_vae2016,zhao2017infovae}.



\paragraph{$\beta$-VAE~\citep{beta_vae2016}} The following objective $\mc{L}_{\beta-\mathrm{VAE}}$ is proposed to learn disentangled features $\zv$:
\begin{align*}
-\bb{E}_{q_\theta(\xv, \zv)}[\log p_\theta(\xv|\zv)] + \beta \bb{E}_{q(\xv)} \left[ \KL(q_\theta(\zv|\xv) \Vert p(\zv)) \right]
\end{align*}
This is equivalent to the following dual form:
\begin{align*}
&\mc{L}_{\beta-\mathrm{VAE}} \\ 
& \equiv 
\bb{E}_{q_\theta(\xv, \zv)}\left[ \log \frac{q_\theta(\xv| \zv)q(\xv)}{p_\theta(\xv|\zv)q_\theta(\xv|\zv)} + \beta \log \frac{q_\theta(\zv|\xv)q_\theta(\zv)}{q_\theta(\zv)p(\zv)} \right] \\
& \begin{array}{ll} \equiv (\beta - 1) I_{q_\theta}(\xv; \zv)  & \mathrm{(primal)} \\
  \quad + \beta \KL(q_\theta(\zv) \Vert p(\zv))) & \mathrm{(consistency)} \\
  \quad + \bb{E}_{q_\theta(\zv)}[\KL(q_\theta(\xv|\zv) \Vert p_\theta(\xv|\zv))] & \end{array}
\end{align*}
where we use $\equiv$ to denote ``equal up to a value that does not depend on $\theta$''. In this case, 
\begin{align*} 
& \alpha_1 = \beta - 1 , \qquad \alpha_2 = 0 \\
& \lambdav = [\beta, 1] \\ 
& \mc{D} = [KL(q_\theta(\zv) \Vert p(\zv))), \bb{E}_{q_\theta(\zv)}[\KL(q_\theta(\xv|\zv) \Vert p_\theta(\xv|\zv))]
\end{align*}
When $\alpha_1 > 0$ or equivalently $\beta > 1$, there is an incentive to minimize mutual information between $\xv$ and $\zv$.

\paragraph{InfoGAN~\citep{infogan2016}} As another example, the InfoGAN objective~\footnote{For conciseness we use $\zv$ to denote structured latent variables, which is represented as $\cv$ in~\citep{infogan2016}.} 
:
\begin{align*}
\mc{L}_{\mathrm{InfoGAN}} = \JS(q(\xv) \Vert p_\theta(\xv)) - \bb{E}_{p_\theta(\xv, \zv)}[\log q_\theta(\zv |\xv)] 
\end{align*}
is equivalent to the following dual form: 
\begin{align*}
& \mc{L}_{\mathrm{InfoGAN}}  \equiv \bb{E}_{p_\theta(\xv, \zv)}[- \log p_\theta(\zv|\xv) + \log p(\zv) \\
& + \log p_\theta(\zv|x) - \log q_\theta(\zv|\xv)] + \JS(q(\xv) \Vert p_\theta(\xv))  \\
& \begin{array}{ll}
\equiv - I_{p_\theta}(\xv; \zv) & \mathrm{(primal)} \\
 \quad + \bb{E}_{p_\theta(\xv)}[\KL(p_\theta(\zv|\xv) \Vert q_\theta(\zv|\xv))]  & \mathrm{(consistency)} \\
 \quad + \JS(q(\xv) \Vert p_\theta(\xv)) & \end{array}
\end{align*}
In this case $\alpha_1 = 0$, and $\alpha_2 = -1 < 0$, the model maximizes mutual information between $\xv$ and $\zv$. 

In fact, all objectives in Table~\ref{tab:params} belong to this class\footnote{Variational Mutual Information Maximization (VMI) is not truly a Lagrangian dual because it does not enforce consistency constraints ($\lambdav=0$). 
}. 
Derivations for additional models can be found in Appendix A.  


\subsection{ENUMERATION OF ALL OBJECTIVES}
The Lagrangian dual form of an objective reveals its mutual information preference ($\alpha_1, \alpha_2$), type of consistency constraints ($\mc{D}$), and weighting of the constraints ($\lambdav$). This suggests that the Lagrangian dual perspective may unify many existing training objectives. We wish to identify and categorize \emph{all} objectives that have Lagrangian dual form as in Eq.\ref{eq:lag_dual}). However, this has two technical difficulties that we proceed to resolve.

\paragraph{1. Equivalence:} Many objectives appear different, but are actually identical for the purposes of optimization (as we have shown). To handle this we characterize ``equivalent objectives'' with a set of pre-specified transformations.  

\begin{definition} 
\label{def:equivalence}
\textbf{Equivalence (Informal):} An objective $\mathcal{L}$ is equivalent to $\mathcal{L}'$ when there exists a constant $C$, so that for all parameters $\theta$, $\mathcal{L}(\theta) = \mathcal{L}'(\theta) + C$. We denote this as $\mathcal{L} \equiv \mathcal{L}'$. 

$\mc{L}$ and $\mc{L}'$ are elementary equivalent if $\mc{L}'$ can be obtained from $\mc{L}$ by applying chain rule or Bayes rule to probabilities in $\mc{L}$, and addition/subtraction of constants $\E_{q(\xv)}[\log q(\xv)]$ and $\E_{p(\zv)}[\log p(\zv)]$. 
\end{definition}
A more formal but verbose definition is in Appendix B, Definition 1. 

Elementary equivalences define simple yet flexible transformations for deriving equivalent objectives. For example, all the transformations in Section 4 (VAE, $\beta$-VAE and InfoGAN) and Appendix A are elementary. This implies that all objectives in Table~\ref{tab:params} are elementary equivalent to a Lagrangian dual function in Eq.(\ref{eq:lag_dual}) . However, these transformations are not exhaustive. For example, 
tranforming $\bb{E}_{p_\theta}[g(\xv)]$ into $\bb{E}_{q_\theta}[g(\xv) p_\theta(\xv) / q_\theta(\xv)]$ via importance sampling is not accounted for, hence the two objectives are not considered to be elementary equivalent. 

\textbf{2. Optimization Difficulty}: Some objectives are easier to evaluate/optimize than others. For example, variational autoencoder training is robust and stable, adversarial training is less stable and requires careful hyper-parameter selection~\citep{kodali2018convergence}, and direct optimization of the log-likelihood $\log p_\theta(\xv)$ is very difficult for latent variable models and almost never used~\cite{grover2018flow}. 

To assign a ``hardness score'' to each objective, we first group the ``terms'' (an objective is a sum of terms) from easy to hard to optimize. An objective belongs to a ``hardness class'' if it cannot be transformed into an objective with easier terms. This is formalized below: 

\begin{definition}
\label{def:tractable_families}
\textbf{Effective Optimization:} We define

1. Likelihood-based terms as the following set
\begin{align*} 
T_1 = \lbrace &\E_{p_\theta(\xv, \zv)}[\log p_\theta(\xv|\zv)],\E_{p_\theta(\xv, \zv)}[\log p_\theta(\xv, \zv)], \\ 
& \E_{p_\theta(\zv)}[\log p(\zv)],\E_{p_\theta(\xv, \zv)}[\log q_\theta(\zv|\xv)] \\  
& \E_{q_\theta(\xv, \zv)}[\log p_\theta(\xv|\zv)],\E_{q_\theta(\xv, \zv)}[\log p_\theta(\xv, \zv)], \\
& \E_{q_\theta(\zv)}[\log p(\zv)], \E_{q_\theta(\xv, \zv)}[\log q_\theta(\zv|\xv)] \rbrace
\end{align*}

2. Unary likelihood-free terms as the following set
\begin{align*} 
T_2 = \lbrace D(q(\xv)\Vert p_\theta(\xv)), D(q_\theta(\zv)\Vert p(\zv)) \rbrace
\end{align*}

3. Binary likelihood-free terms as the  following set
\begin{align*}
T_3 = \lbrace D(q_\theta(\xv, \zv)\Vert p_\theta(\xv, \zv)) \rbrace
\end{align*}

where each $D$ can be $f$-divergence, Jensen Shannon divergence, Wasserstein distance, or Maximum Mean Discrepancy.
An objective $\mathcal{L}$ is likelihood-based computable if $\mathcal{L}$ is elementary equivalent to some $\mathcal{L}'$ that is a linear combination of elements in $T_1$; unary likelihood-free computable if $\mathcal{L}'$ is a linear combination of elements in $T_1 \cup T_2$; binary likelihood-free computable if $\mathcal{L}'$ is a linear combination of elements in $T_1 \cup T_2 \cup T_3$.
\end{definition}

The rationale of this categorization is that elements in $T_1$ can be estimated by Monte-Carlo estimators and optimized by stochastic gradient descent effectively in practice (with low bias and variance)~\citep{autoencoding_variational_bayes2013,variational_dbn_stochastic_bp2014}. 
In contrast, elements in $T_2$ are optimized by likelihood-free approaches such as adversarial training~\citep{generative_adversarial_nets2014} or kernelized methods such as MMD~\citep{mmd_statistics2007} or Stein variational gradient~\citep{stein_variational2016}. These optimization procedures are known to suffer from stability problems~\citep{wgan2017} or cannot handle complex distributions in high dimensions~\citep{ramdas2015decreasing}. Finally, elements in $T_3$ are over both $\xv$ and $\zv$, and they are empirically shown to be even more difficult to optimize~\citep{alice2017}. We do not include terms such as $\Eb_{q(x)} [\log p_\theta(x)]$ because they are seldom feasible to compute or optimize for latent variable generative models. 

Now we are able to fully characterize all Lagrangian dual objectives in Eq.(~\ref{eq:lag_dual}) that are likelihood-based / unary likelihood free / binary likelihood free computable in Table~\ref{tab:params}. 


In addition, Table~\ref{tab:params} contains essentially \emph{all} possible models for each optimization difficulty class in Definition~\ref{def:tractable_families}. This is shown in the following theorem (informal, formal version and proof in Appendix B, Theorem 3,4,5)

\begin{theorem}
\textbf{Closure theorem (Informal)}: Denote a Lagrangian objectives in the form of Equation~\ref{eq:lag_dual} where all divergences are $\KL$ a KL Lagrangian objective. Under elementary equivalence defined in Definition~\ref{def:equivalence}, 

\textbf{1)} Any KL Lagrangian objective that is elementary equivalent to a likelihood based computable objective is equivalent to a linear combination of VMI and $\beta$-VAE.

\textbf{2)} Any KL Lagrangian objective that is elementary equivalent to a unary likelihood computable objective is equivalent to a linear combination of InfoVAE and InfoGAN.

\textbf{3)} Any KL Lagrangian objective that is elementary equivalent to a binary likelihood computable objective is equivalent to a linear combination of ALICE, InfoVAE and InfoGAN. 
\end{theorem}

We also argue in the Appendix (without formal proof) that this theorem holds for other divergences including $\MMD$, $D_{\mathrm{W}}$, $D_{f}$ or $D_{\mathrm{JS}}$. 

Intuitively, this suggests a rather negative result: if a new latent variable model training objective 
contains mutual information preference and consistency constraints (defined through $\KL$, $\MMD$, $D_{\mathrm{W}}$, $D_{f}$ or $D_{\mathrm{JS}}$), and this objective 
can be effectively optimized as in Definition~\ref{def:equivalence} and Definition~\ref{def:tractable_families}, then this objective is a linear combination of existing objectives. Our limitation is that we are restricted to elementary transformations and the set of terms defined in Definition~\ref{def:tractable_families}. To derive new training objectives, we should consider new transformations, non-linear combinations and/or new terms. 






%% file: 4_tractable.tex
\section{DUAL OPTIMIZATION FOR LATENT VARIABLE GENRATIVE MODELS}
While existing objectives for latent variable generative models have dual form in Equation~\ref{eq:lag_dual}, they are not solving the dual problem exactly because the Lagrange multipliers $\lambdav$ are predetermined instead of optimized.
In particular, if we can show strong duality, the optimal solution to the dual is also an optimal solution to the primal~\citep{boyd2004convex}. However if the Lagrange multipliers are fixed, this property is lost, and the parameters $\theta$ obtained via dual optimization may be suboptimal for $\min_\theta f(\theta)$, or violate the consistency conditions $\mathcal{D} = \zerov$. 

\subsection{RELAXATION OF CONSISTENCY CONSTRAINTS}

This observation motivates us to directly solve the dual optimization problem where we also \emph{optimize} the Lagrange multipliers. 
\[ \max_{\lambdav \geq 0} \min_\theta f(\theta) + \lambdav^T \mc{D} \] 
Unfortunately, this is usually impractical because the consistency constrains are difficult to satisfy when the model has finite capacity, so in practice the primal optimization problem is actually infeasible and $\lambdav$ will be optimized to $+\infty$.

One approach to this problem is to use relaxed consistency constraints, where compared to Eq.(\ref{eq:primal}) we require consistency up to some error $\epsilonv > 0$: 
\begin{align}
\min_\theta f(\theta) \quad \text{subject to} \quad \mc{D} \leq \epsilonv \label{eq:soft-primal}
\end{align}
For a sufficiently large $\epsilonv$, the problem is feasible. 
This has the corresponding dual problem:
\begin{align}
\max_{\lambdav \geq 0} \min_\theta f(\theta) + \lambdav^\top (\mc{D} - \epsilonv) \label{eq:eps-dual}
\end{align}

When $\lambdav$ is constant (instead of maximized), Equation~\ref{eq:eps-dual} still reduces to existing latent variable generative modeling objectives since $\lambdav^\top \epsilonv$ is a constant, so the objective simply becomes 
\[ \min_\theta f(\theta) + \lambdav^T \mc{D} + \text{constant} \] 

In contrast, we propose to find $\lambdav^*, \theta^*$ that optimize the Lagrangian dual in Eq.(\ref{eq:eps-dual}). If we additionally have strong duality, $\theta^*$ is also the optimal solution to the primal problem in Eq.(\ref{eq:soft-primal}). 

\subsection{STRONG DUALITY WITH MUTUAL INFORMATION OBJECTIVES}
This section aims to show that strong duality for Eq.(\ref{eq:soft-primal}) holds in distribution space if we replace mutual informations in $f$ with upper and lower bounds. 
We prove this via Slater's condition~\citep{boyd2004convex}, which has three requirements: 1. $\forall D \in \mathcal{D}$, $D$ is convex in $\theta$; 2. $f(\theta)$ is convex for $\theta \in \Theta$;  3. the problem is strictly feasible: $\exists \theta$ s.t. $\mathcal{D} < \epsilonv$. We propose weak conditions to satisfy all three in distribution space, so strong duality is guaranteed. 

For simplicity we focus on discrete $\mathcal{X}$ and $\mathcal{Z}$. We parameterize $q_\theta(\zv \vert \xv)$ with a parameter matrix $\theta^q \in \mathbb{R}^{|\mathcal{X}||\mathcal{Z}|}$ (we add the superscript $q$ to distinguish parameters of $q_\theta$ from that of $p_\theta$) where
\begin{align*}
q_{\theta}(\zv=j \vert \xv=i) = \theta^q_{ij}, \forall i \in \mathcal{X}, j \in \mathcal{Z}  \numberthis \label{eq:theta_notation}
\end{align*}
The only restriction is that $\theta^q$ must correspond to valid conditional distributions. More formally, we require that $\theta^q \in \Theta^q$, where
\begin{align*}
\Theta^q = \left\lbrace \theta^q \in \mathbb{R}^{|\mathcal{X}||\mathcal{Z}|} \ s.t. \ 0 \leq \theta^q_{ij} \leq 1, \sum_j \theta^q_{ij} = 1 \right\rbrace \numberthis \label{eq:theta_set}
\end{align*}
Similarly we can define $\theta^p \in \Theta^p$ for $p_\theta$. We still use 
\begin{align*}
\theta = [\theta^q, \theta^p], \quad \Theta = \Theta^q \times \Theta^p \numberthis \label{eq:theta_both}
\end{align*}
to denote both sets of parameters.

\textbf{1) Constraints $D \in \mathcal{D}$ are convex}: 
We show that some divergences used in existing models are convex in distribution space.
\begin{lemma}[Convex Constraints (Informal)]
\label{lemma:convex_constraint}
$\KL$, $\MMD$, or $D_f$ over any marginal distributions on $x$ or $z$ or joint distributions on $(x, z)$ are convex with respect to $\theta \in \Theta$ as defined in Eq.(\ref{eq:theta_both}).
\end{lemma}

Therefore if one only uses these convex divergences, the first requirement for Slater's condition is satisfied. 

\textbf{2) Convex Bounds for $f(\theta)$}:
$f(\theta) = \alpha_1 I_{q_\theta}(x; z) + \alpha_2 I_{p_\theta}(x; z)$ is not itself guaranteed to be convex in general. 
However we observe that mutual information has a convex upper bound, and a concave lower bound, which we denote as $\overline{I}_{q_\theta}$ and $\underline{I}_{q_\theta}$ respectively:

\begin{align*}
& I_{q_\theta}(\xv; \zv) \numberthis \label{eq:bounds} \\
&\begin{array}{ll}
=\bb{E}_{q(\xv)}[\KL(q_\theta(\zv|\xv)\Vert p(\zv))] & \small{\text{convex upper bound } \overline{I}_{q_\theta}} \\
 \quad - \KL(q_\theta(\zv) \Vert p(\zv)) & \small{\text{bound gap } \overline{I}_{q_\theta} - I_{q_\theta}}
\end{array}
\\
&\begin{array}{ll}
= \bb{E}_{q_\theta(\xv, \zv)}[\log p_\theta(\xv|\zv)] + H_q(\xv) & \small{\text{concave lower bound } \underline{I}_{q_\theta}}  \\ 
\quad + \Eb_{p(\zv)} \KL(q(\xv|\zv) \Vert p_\theta(\xv|\zv))  & \small{\text{bound gap } I_{q_\theta}- \underline{I}_{q_\theta} }
 \end{array} 
\end{align*}
The convexity/concavity of these bounds is shown by the following lemma, which we prove in the appendix
\begin{lemma}[Convex/Concave Bounds]
\label{lemma:convex_bound}
$\overline{I}_{q_\theta}$ is convex with respect to $\theta \in \Theta$ as defined in Eq.(\ref{eq:theta_both}), and $\underline{I}_{q_\theta}$ is concave with respect to $\theta \in \Theta$.
\end{lemma}

A desirable property of these bounds is that if we look at the bound gaps (difference between bound and true value) in Eq.(\ref{eq:bounds}), they are $0$ if the consistency constraint is satisfied (i.e., $p_\theta(\xv, \zv) = q_\theta(\xv, \zv)$). They will be tight (bound gaps are small) when consistency constraints are approximately satisfied (i.e., $p_\theta(\xv, \zv) \approx q_\theta(\xv, \zv)$). In addition we also denote identical bounds for $I_{p_\theta}$ as $\overline{I}_{p_\theta}$ and $\underline{I}_{p_\theta}$ Similar bounds for mutual information have been discussed in~\citep{alemi2017fixing}.

\textbf{3) Strict Feasibility}: the optimization problem has non empty feasible set, which we show in the following lemma:
\begin{lemma}[Strict Feasibility]
\label{lemma:non_empty_interior}
For discrete $\mathcal{X}$ and $\mathcal{Z}$, and $\epsilonv > 0$, $\exists \theta \in \Theta$ such that $\mathcal{D} < \epsilonv$. 
\end{lemma}


Therefore we have shown that for convex/concave upper and lower bounds on $f$, all three of Slater's conditions are satisfied, so strong duality holds. 
We summarize this in the following theorem. 
\begin{theorem}[Strong Duality]
\label{thm:strong_duality}
If $\mc{D}$ contains only divergences in Lemma~\ref{lemma:convex_constraint}, then for all $\epsilon > 0$:

If $\alpha_1, \alpha_2 \geq 0$ strong duality holds for the following problems: 
\begin{equation}
\min_{\theta \in \Theta} \alpha_1 \overline{I}_{q_\theta} + \alpha_2 \overline{I}_{p_\theta} \quad \text{subject to} \quad \mc{D} \leq \epsilon \label{eq:primal-upper}
\end{equation}

If $\alpha_1, \alpha_2 \leq 0$, strong duality holds for the following problem
\begin{equation}
\min_{\theta \in \Theta} \alpha_1 \underline{I}_{q_\theta} + \alpha_2 \underline{I}_{p_\theta} \quad \text{subject to} \quad \mc{D} \leq \epsilon \label{eq:primal-lower}
\end{equation}
\end{theorem}



\subsection{DUAL OPTIMIZATION}
Because the problem is convex in distribution space and satisfies Slater's condition, the $\theta^*$ that achieves the saddle point
\begin{align*}
\lambdav^\star, \theta^\star = {\arg\max}_{\lambdav \geq 0} {\arg\min}_\theta f(\theta) + \lambdav^T (\mc{D} - \epsilonv) \numberthis \label{eq:saddle_point}
\end{align*}
is also a solution to the original optimization problem Eq.(\ref{eq:soft-primal})~\citep{boyd2004convex}(Chapter 5.4). In addition the max-min problem Eq.(\ref{eq:saddle_point}) is convex with respect to $\theta$ and concave (linear) with respect to $\lambdav$, so one can apply iterative gradient descent/ascent over $\theta$ (minimize) and $\lambdav$ (maximize) and achieve stable convergence to saddle point~\citep{holding2014convergence}. 
We describe the iterative algorithm in Algorithm~\ref{alg:dual}.

\begin{algorithm}[t]
   \caption{Dual Optimization for Latent Variable Generative Models}
   \label{alg:dual}
\begin{algorithmic}
   \STATE {\bfseries Input:} Analytical form for $p(\zv)$ and samples from $q(\xv)$; constraints $\mc{D}$; $\alpha_1, \alpha_2$ that specify maximization / minimization of mutual information; $\epsilonv > 0$ which specifies the strength of constraints; step size $\rho_\theta$, $\rho_\lambda$ for $\theta$ and $\lambdav$.
   
   \STATE {\bfseries Output:} $\theta$ (parameters for $p_\theta(\xv|\zv)$ and $q_\theta(\zv|\xv)$).
   
   \hspace{-1em}\hrulefill
   
   \STATE Initialize $\theta$ randomly
   \STATE Initialize the Lagrange multipliers $\lambdav := 1$
   \IF{$\alpha_1, \alpha_2 > 0$}
   \STATE $f(\theta) \leftarrow \alpha_1 \overline{I}_{q_\theta} + \alpha_2 \overline{I}_{p_\theta}$
   
   \ELSE
   \STATE $f(\theta) \leftarrow \alpha_1 \underline{I}_{q_\theta} + \alpha_2 \underline{I}_{p_\theta}$
   
   \ENDIF
   \FOR{$t=0, 1, 2, \ldots $}
   \STATE $\theta \leftarrow \theta - \rho_\theta (\nabla_\theta f(\theta) + \lambdav^\top \nabla_\theta \mc{D})$
   \STATE $\lambdav \leftarrow \lambdav + \rho_\lambda (\mc{D} - \epsilonv)$
   \ENDFOR
\end{algorithmic}
\end{algorithm}

In practice, we do not optimize over distribution space and $\lbrace p_\theta(\xv|\zv) \rbrace, \lbrace q_\theta(\zv|\xv) \rbrace$ are some highly complex and non-convex families of functions. 
We show in the experimental section that this scheme is stable and effective despite non-convexity. 

%% file: 5_lagvae.tex
\section{LAGRANGIAN VAE}
In this section we consider a particular instantiation of the general dual problem proposed in the previous section. Consider the following primal problem, with $\alpha_1 \in \bb{R}$:
\begin{align}
 \ \min_\theta & \ \alpha_1 I_{q_\theta}(\xv; \zv) \label{eq:lagvae-primal} \\
\text{subject to } & \ \KL(q_\theta(\xv, \zv) \Vert p_\theta(\xv, \zv))) \leq \epsilon_1 \nonumber \\
& \ \MMD(q_\theta(\zv)\Vert p(\zv)) \leq \epsilon_2 \nonumber
\end{align}





For mutual information minimization / maximization, we respectively replace the (possibly non-convex) mutual information by upper bound $\overline{I}_{q_\theta}$ if $\alpha_1 \geq 0$ and lower bound $\underline{I}_{q_\theta}$ if $\alpha_1 < 0$. The corresponding dual optimization problem can be written as:
\begin{align*}
\max_{\lambdav \geq 0} \min_\theta  \left\lbrace \begin{array}{ll} 
 \alpha_1 \overline{I}_{q_\theta} + \lambdav^\top (\mc{D}_{\mathrm{InfoVAE}} - \epsilonv), & \alpha_1 \geq 0 \\
 \alpha_1 \underline{I}_{q_\theta} + \lambdav^\top (\mc{D}_{\mathrm{InfoVAE}} - \epsilonv), & \alpha_1 < 0 
 \end{array} \right. \numberthis  \label{eq:lagvae}
\end{align*}
where $\epsilonv = [\epsilon_1, \epsilon_2]$,  $\lambdav = [\lambda_1, \lambda_2]$ and 
\begin{align*}
\mc{D}_{\mathrm{InfoVAE}} = [&\KL(q_\theta(\xv, \zv) \Vert p_\theta(\xv, \zv))), \\
&\MMD(q_\theta(\zv)\Vert p(\zv))]
\end{align*}
We call the objective in~\ref{eq:lagvae} Lagrangian (Info)VAE (LagVAE). Note that setting a constant $\lambdav$ for the dual function recovers the InfoVAE objective~\citep{zhao2017infovae}. 
By Theorem~\ref{thm:strong_duality} strong duality holds for this problem and finding the max-min saddle point of LagVAE in Eq.(\ref{eq:lagvae}) is identical to finding the optimal solution to original problem of Eq.(\ref{eq:lagvae-primal}). 

The final issue is choosing the $\epsilonv$ hyper-parameters so that the constraints are feasible. This is non-trivial since selecting $\epsilonv$ that describe feasible constraints 
depends on the task and model structure. 
We introduce a general strategy that is effective in all of our experiments. First we learn a parameter $\theta^*$ that satisfies the consistency constraints ``as well as possible'' without considering mutual information maximization/minimization. Formally this is achieved by the following optimization (for any choice of $\lambdav > 0$), 
\begin{align*}
\theta^* = \arg\min_\theta \quad \lambdav^T  \mc{D}_{\mathrm{InfoVAE}} \numberthis \label{eq:find_feasible}
\end{align*}
This is the original training objective for InfoVAE with $\alpha_1 = 0$ and can be optimized by 
\begin{align*}
 \min_\theta & \lambdav^T  \mc{D}_{\mathrm{InfoVAE}} \\
= & \lambda_1 \KL(q_\theta(\xv, \zv) \Vert p_\theta(\xv, \zv))) + \lambda_2 \MMD(q_\theta(\zv)\Vert p(\zv)) \\
\equiv & \lambda_1 \mathcal{L}_\mathrm{ELBO}(\theta) + \lambda_2 \MMD(q_\theta(\zv)\Vert p(\zv)) \numberthis \label{eq:infovae_without_mi}
\end{align*}
where $\mathcal{L}_\mathrm{ELBO}(\theta)$ is the evidence lower bound defined in Eq.(\ref{eq:vae_computable_form}). Because we only need a rough estimate of how well consistency constraints can be satisfied, the selection of weighing $\lambda_1$ and $\lambda_2$ is unimportant. The recommendation in \citep{zhao2017infovae} works well in all our experiments ($\lambda_1=1, \lambda_2=100$). 

Now we introduce a ``slack'' to specify how much we are willing to tolerate consistency error to achieve higher/lower mutual information. Formally, we define $\hat{\epsilonv}$ as the divergences $\mc{D}_{\mathrm{InfoVAE}}$ evaluated at the above $\theta^*$. Under this $\hat{\epsilonv}$ the following constraint must be feasible (because $\theta^*$ is a solution):
\[ \mc{D}_{\mathrm{InfoVAE}} \leq \hat{\epsilonv}  \]
Now we can safely set $\epsilonv = \gammav + \hat{\epsilonv}$, where $\gammav > 0$, and the constraint 
\[ \mc{D}_{\mathrm{InfoVAE}} \leq \epsilonv  \]
must still be feasible (and strictly feasible). $\gammav$ has a very nice intuitive interpretation: it is the ``slack'' that we are willing to accept. Compared to tuning $\alpha_1$ and $\lambdav$ for InfoVAE, tuning $\gammav$ is much more interpretable: we can anticipate the final consistency error before training.

Another practical consideration is that the one of the constraints $\KL(q_\theta(\xv, \zv) \Vert p_\theta(\xv, \zv))$ is difficult to estimate. However, we have
\begin{align}
\KL(q_\theta(\xv, \zv) \Vert p_\theta(\xv, \zv)) = -\mathcal{L}_\mathrm{ELBO} - H_q(\xv) \nonumber
\end{align}
where $\mathcal{L}_\mathrm{ELBO}$ is again, the evidence lower bound in Eq.(\ref{eq:vae_computable_form}) of Section 2, and $H_q(\xv)$ is the entropy of the true distribution $q(x)$. $\mathcal{L}_\mathrm{ELBO}$ is empirically easy to estimate, and $H_q(\xv)$ is a constant irrelevant to the optimization problem. The optimization problem is identical if we replacing the more difficult constraint $\KL(q_\theta(\xv, \zv) \Vert p_\theta(\xv, \zv)) \leq \epsilon_1$ with the easier-to-optimize/estimate constraint $-\mathcal{L}_\mathrm{ELBO} \leq \epsilon_1'$ (where $\epsilon_1' = \epsilon_1 + H_q(\xv)$). In addition, $\epsilon_1'$ can be selected by the technique in the previous paragraph. 

%% file: 5_evaluation.tex
\section{EXPERIMENTS}
We compare the performance of \textbf{LagVAE}, where we learn $\lambdav$ automatically, and \textbf{InfoVAE}, where we set $\lambdav$ in advance (as hyperparameters).
Our primal problem is to find solutions that maximize / minimize mutual information under the consistency constraints. Therefore, we consider two performance metrics:
\begin{itemize}
\item $I_q(\xv, \zv)$ the mutual information between $\xv$ and $\zv$. We can estimate the mutual information via the identity: 
\begin{align*}
I_q(\xv; \zv) = \bb{E}_{q_\theta(\xv, \zv)}\left[ \log q_\theta(\zv|\xv) - \log q_\theta(\zv) \right] \numberthis \label{eq:mi_estimation}
\end{align*}
where we approximate $q_\theta(\zv)$ with a kernel density estimator.
\item the consistency divergences $\KL(q_\theta(\xv, \zv) \Vert p_\theta(\xv, \zv))$ and 
$\MMD(q_\theta(\zv)\Vert p(\zv))$. As stated in Section 6, we replace $\KL(q_\theta(\xv, \zv) \Vert p_\theta(\xv, \zv))$ with the evidence lower bound $\mc{L}_{\mathrm{ELBO}}$.  
\end{itemize}

\begin{figure*}
\centering
\begin{tabular}{ccc}
\includegraphics[width=0.35\textwidth]{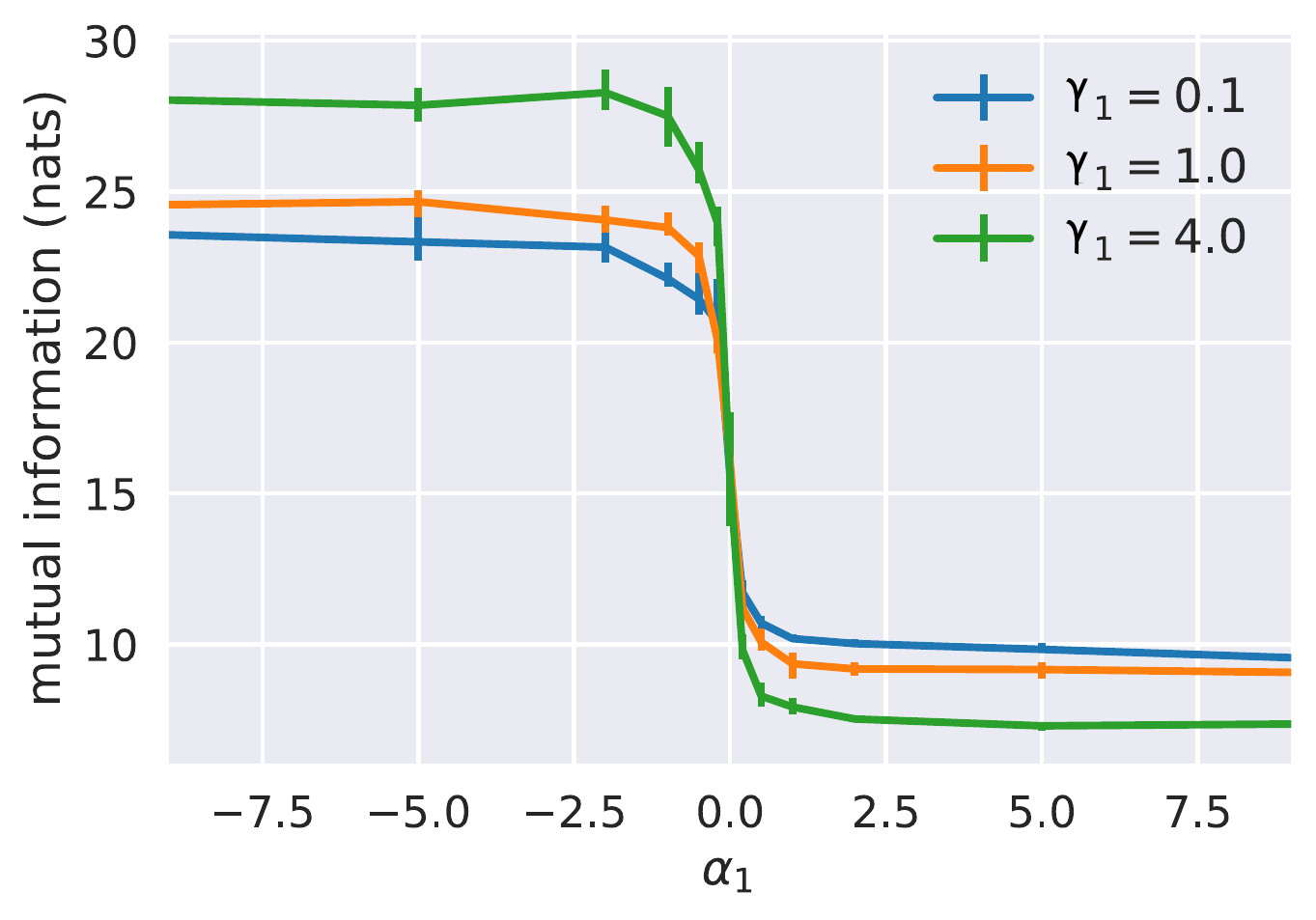} &
\hspace{5em} &
\includegraphics[width=0.30\textwidth]{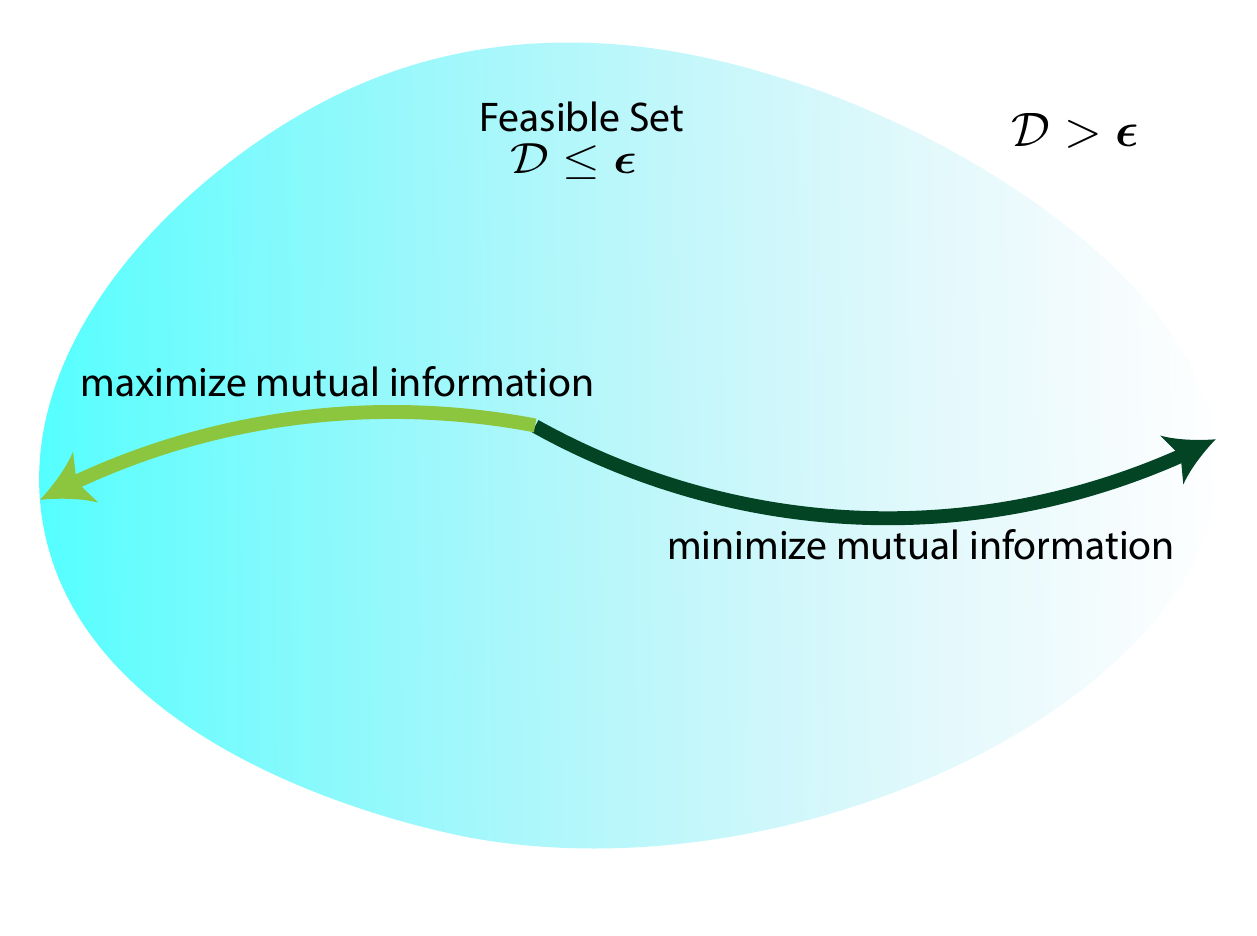} 
\end{tabular}
\caption{\textbf{Left:} Effect of $\alpha_1$ and $\gamma_1$ on the primal objective (mutual information). When $\alpha_1$ is positive we minimize mutual information within the feasible set, and when $\alpha_1$ is negative we maximize mutual information. When $\alpha_1$ is zero the preference is undetermined, and mutual information varies depending on initialization. Note that mutual information does not depend on the absolute value of $\alpha_1$ but only on its sign. \textbf{Right:} An illustration of this effect. Lagrangian dual optimization finds the maximum/minimum mutual information solution in the feasible set $\mc{D} \leq \epsilonv$.}
\label{fig:alpha_vs_mi}
\end{figure*}



In the remainder of this section we demonstrate the following empirical observations:
\begin{itemize}
\item LagVAE reliably maximizes/minimizes mutual information without violating the consistency constraints. InfoVAE, on the other hand, makes unpredictable and task-specific trade-offs between mutual information and consistency. 
\item LagVAE is Pareto optimal, as no InfoVAE hyper-parameter choice is able to achieve both better mutual information and better consistency (measured by $\MMD$ and $\ELBO$) than LagVAE. 
\end{itemize}

\subsection{VERIFICATION OF DUAL OPTIMIZATION}
We first verify that LagVAE reliably maximizes/minimizes mutual information subject to consistency constraints. We train LagVAE on MNIST according to Algorithm~\ref{alg:dual}. $\epsilonv$ is selected according to Section 6, where we first compute $\hat{\epsilonv} = (\hat{\epsilon}_1, \hat{\epsilon}_2)$ without information maximization/minimization by Eq.(\ref{eq:infovae_without_mi}). Next we choose slack variables $\gammav = (\gamma_1, \gamma_2)$, and set $\epsilonv = \hat{\epsilonv} + \gammav$. For $\gamma_1$ we explore values from 0.1 to 4.0, and for $\gamma_2$ we use the fixed value $0.5 \hat{\epsilon}_2$. 

The results are shown in Figure~\ref{fig:alpha_vs_mi}, where mutual information is estimated according to Eq.(\ref{eq:mi_estimation}). For any given slack $\gammav$, setting $\alpha_1$ to positive values and negative values respectively minimizes or maximizes the mutual information within the feasible set $\mc{D} \leq \epsilonv$. In particular, the absolute value of $\alpha_1$ does not affect the outcome, and only the sign matters. This is consistent with the expected behavior (Figure~\ref{fig:alpha_vs_mi} Left) where the model finds the maximum/minimum mutual information solution within the feasible set. 

\begin{figure}
\centering
\begin{subfigure}{0.45\textwidth}
\centering
\includegraphics[width=0.94\textwidth,trim=9px 3px 0 3px,clip]{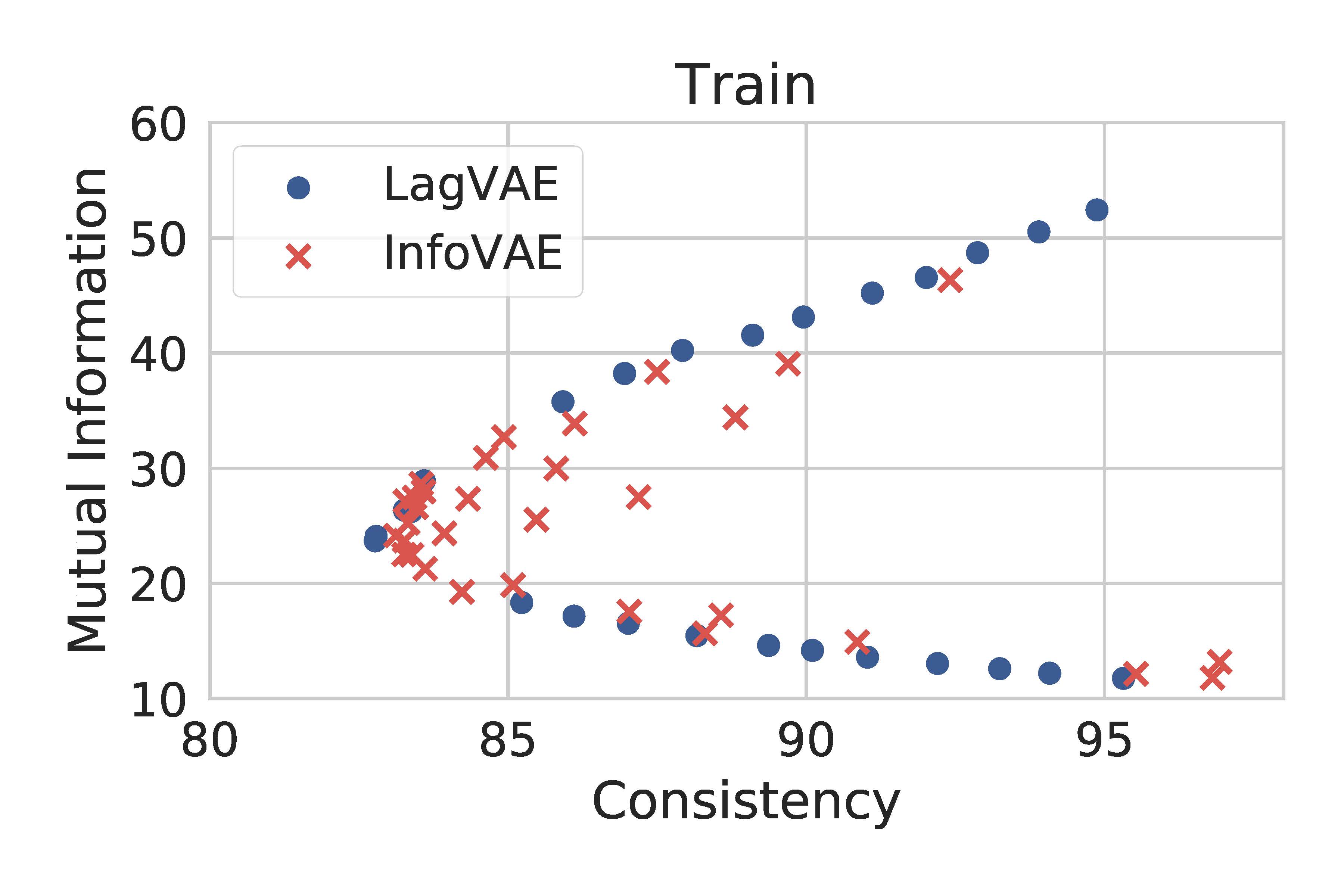}
\end{subfigure}
~
\begin{subfigure}{0.45\textwidth}
\centering
\includegraphics[width=\textwidth]{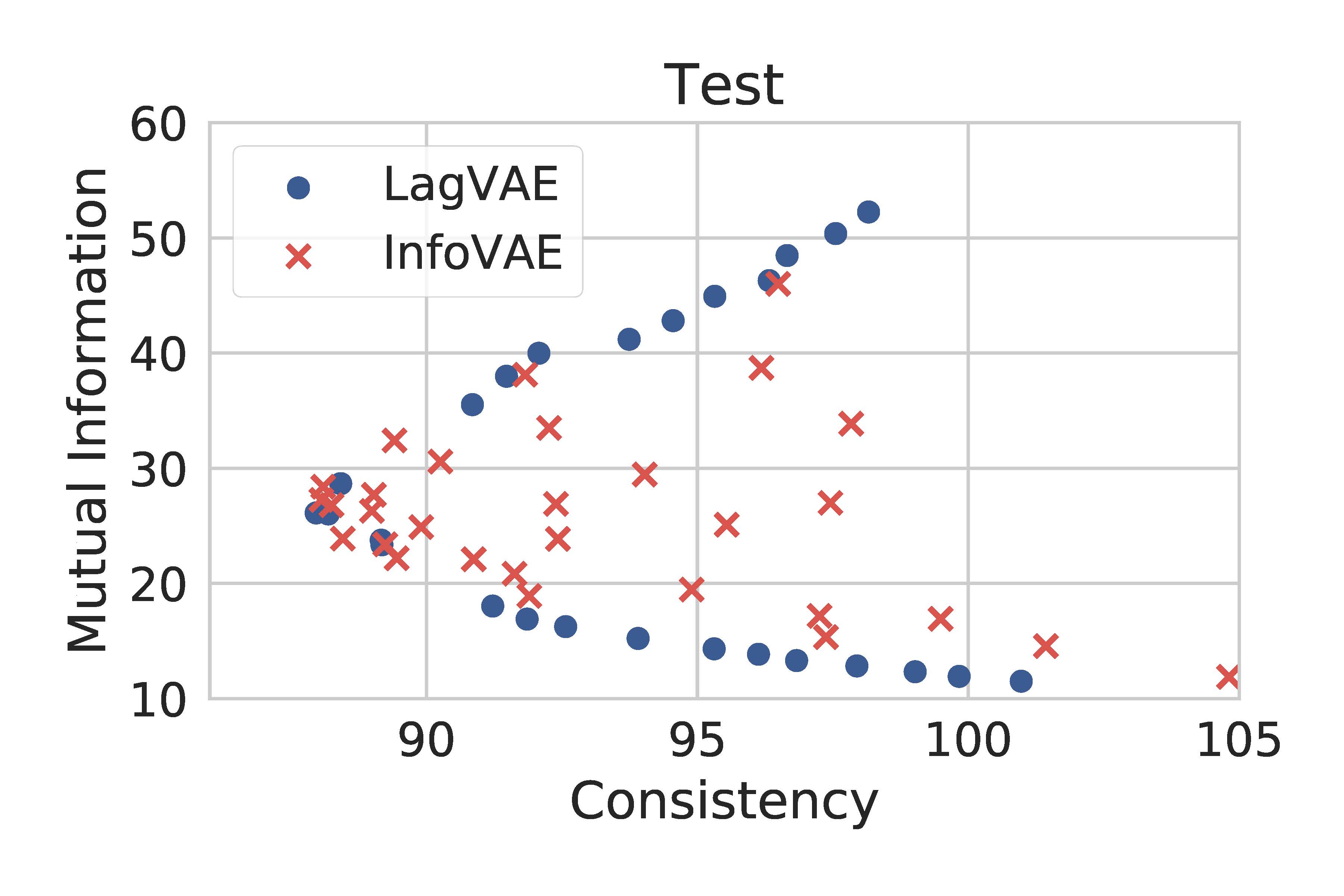}
\end{subfigure}
\caption{LagVAE Pareto dominates InfoVAE with respect to Mutual information and consistency ($\ELBO$ values) on train (top) and test (bottom) set. Each point is the outcome of one hyper-parameter choice for LagVAE / InfoVAE. When we maximize mutual information ($\alpha_1 < 0$), for any given $\ELBO$ value, LagVAE always achieve similar or larger mutual information; when we minimize mutual information ($\alpha_1 > 0$), for any given ELBO value, LagVAE always achieve similar or smaller mutual information.}
\label{fig:pareto}
\end{figure}

\subsection{VERIFICATION OF PARETO IMPROVEMENTS}
\label{sec:exp-pareto}
In this section we verify Pareto optimality of LagVAE. 
We evaluate LagVAE and InfoVAE on the MNIST dataset with a wide variety of hyper-parameters.
For LagVAE, we set $\epsilon_1$ for $\ELBO$ to be $\{83, 84, \ldots, 95\}$ and $\epsilon_2$ for $\MMD$ to be $0.0005$. For InfoVAE, we set $\alpha \in \{1, -1\}$, $\lambda_1 \in \{1, 2, 5, 10\}$ and $\lambda_2 \in \{1000, 2000, 5000, 10000\}$~\footnote{Code for this set of experiments is available at \url{https://github.com/ermongroup/lagvae}}. 

Figure~\ref{fig:pareto} plots the mutual information and $\ELBO$ achieved by both methods. Each point is the outcome of one hyper-parameter choice of LagVAE / InfoVAE. 
Regardless of the hyper-parameter choice of both models, no InfoVAE hyper-parameter lead to better performance on both mutual information and $\ELBO$ on the training set. This is expected because LagVAE always finds the maximum/minimum mutual information solution out of all solutions with given consistency value. The same trend is true even on the test set, indicating that it is not an outcome of over-fitting.   

%% file: 8_conclusion.tex
\section{CONCLUSION}
Many existing objectives for latent variable generative modeling are Lagrangian dual functions of the same type of constrained optimization problem with fixed Lagrangian multipliers. This allows us to explore their statistical and computational trade-offs, and characterize all models in this class. Moreover, we propose a practical dual optimization method that optimizes both the Lagrange multipliers and the model parameters, allowing us to specify interpretable constraints and achieve Pareto-optimality empirically. 

In this work, we only considered Lagrangian (Info)VAE, but the method is generally applicable to other Lagrangian dual objectives. In addition we only considered mutual information preference. Exploring different preferences is a promising future directions. 




%% file: 9_appendix.tex
\FloatBarrier
\newpage
\onecolumn
\appendix
\section{Relationship with Existing Models}
\label{sec:relationship}
This section is an extension for Section 4. We apply elementary equivalent transformations (defined in Definition 1) to existing objectives and convert them to Lagrangian dual form of Eq.(\ref{eq:lag_dual}).

\paragraph{InfoGAN~\citep{infogan2016}} The InfoGAN objective\footnote{The original InfoGAN applies information maximization only to a subset of latent variables $z$.} is
\[ \mathcal{L}_{\mathrm{InfoGAN}} = -\Eb_{\pt(\xv, \zv)}[\log \qt(\zv|\xv)] + \JS(\pt(\xv)\Vert q(\xv)) \]
and this can be converted to
\begin{align*}
\begin{array}{lll} 
\mathcal{L}_{\mathrm{InfoGAN}} &= -I_{\pt}(\xv; \zv) & \text{ primal} \\ 
& + \E_{\pt(\xv)}[\KL(\pt(\zv|\xv)\Vert \qt(\zv|\xv))] + \JS(\pt(\xv)\Vert q(\xv)) & \text{ consistency}
\end{array}
\end{align*}
where $\alpha_1 = 0$, $\alpha_2 = -1$, and the Jensen-Shannon divergence $\JS$ is approximately optimized in a likelihood-free way with adversarial training. This equivalent formulation highlights the mutual information maximization property of InfoGAN. It also reveals that the InfoGAN objective encourages correct inference because of the $\KL(\pt(\zv|\xv)\Vert \qt(\zv|\xv))$ term.

\paragraph{Adversarial Autoencoder~\citep{adversarial_autoencoder2015}}
Adversarial Autoencoders 
\[ \mathcal{L}_{\mathrm{AAE}} = -\Eb_{\qt(\xv, \zv)}[\log \pt(\xv|\zv)] + \JS(\qt(\zv)\Vert p(\zv)) \] 
and this can be converted to:
\begin{align*}
\begin{array}{lll}
\mathcal{L}_{\mathrm{AAE}} &= -I_{\qt}(\xv; \zv) & \text{ primal} \\ 
& + \E_{\qt(\zv)}[\KL(\qt(\xv|\zv)\Vert \pt(\xv|\zv))] + \JS(\qt(\zv)\Vert p(\zv)) & \text{ consistency}
\end{array}
\end{align*}
For this objective $\alpha_1 = -1$ so mutual information is maximized. 

\paragraph{InfoVAE~\citep{zhao2017infovae}}
InfoVAE can be converted to:
\begin{align*}
\begin{array}{lll}
\mathcal{L}_{\mathrm{InfoVAE}} &= -\beta I_{\qt}(\xv; \zv) & \text{ primal} \\
& + \E_{\qt(\zv)}[\KL(\qt(\xv|\zv)\Vert \pt(\xv|\zv))] + (1-\beta) \KL(\qt(\zv) \Vert p(\zv)) + (\lambda - \beta-1) D(\qt(\zv)\Vert p(\zv)) & \text{ consistency}
\end{array}
\end{align*}
where $D$ is any divergence that can be optimized with likelihood-free approaches. For this objective $\beta$ can be selected to maximize/minimize mutual information. 

\paragraph{ALI/BiGAN~\citep{ali2016,bigan2016,alice2017}} ALI/BiGAN corresponds to direct minimization of $\JS(\qt(\xv, \zv) \Vert \pt(\xv, \zv))$, while ALICE~\citep{alice2017} has the following Lagrangian form:
\begin{align*}
\begin{array}{lll}
\mathcal{L}_{\mathrm{ALICE}} &= -I_{\qt}(\xv; \zv) & \text{ primal} \\ 
& + \E_{\qt(\zv)}[\KL(\qt(\xv|\zv)\Vert \pt(\xv|\zv))] + \JS(\qt(\xv, \zv) \Vert \pt(\xv, \zv)) & \text{ consistency} 
\end{array}
\end{align*}
For this objective $\alpha_1=-1$ so mutual information is maximized. 

\paragraph{CycleGAN/DiscoGAN~\citep{cyclegan2017,discogan2017}} can be written as 
\begin{align*}
\begin{array}{lll}
\mathcal{L}_{\mathrm{CycleGAN}} &= -I_{\qt}(\xv; \zv) -I_{\pt}(\xv; \zv) & \text{ primal} \\
& + \E_{\qt(\zv)}[\KL(\qt(\xv|\zv) \Vert \pt(\xv|\zv))] + 
\E_{\pt(\xv)}[\KL(\pt(\zv|\xv) \Vert \qt(\zv|\xv))]  & \text{ consistency} \\
& + \JS(q(\xv) \Vert \pt(\xv)) + \JS(\qt(\zv) \Vert p(\zv)) 
\end{array}
\end{align*}
For this objective $\alpha_1 = -1$, $\alpha_2 = -1$, so mutual information is maximized. In addition the objective encourages correct inference in both direction by matching $\qt(\xv|\zv)$ to $\pt(\xv|\zv)$ and matching $\pt(\zv|\xv)$ to $\qt(\zv|\xv)$. 

\paragraph{AS-VAE\citep{pu2017adversarial}} can be written as 
\begin{align*}
\begin{array}{lll}
\mc{L}_{\mathrm{AS-VAE}} &= -I_{\qt}(\xv; \zv) - I_{\pt}(\xv; \zv) & \text{ primal} \\
& + \KL(\qt(\xv, \zv) \Vert \pt(\xv, \zv)) + \KL(\pt(\xv, \zv) \Vert \qt(\xv, \zv)) & \text{ consistency} \\
& + \E_{p(\zv)}[\KL(\pt(\xv|\zv) \Vert \qt(\xv|\zv))] + \E_{q(\xv)}[\KL(\qt(\zv|\xv) \Vert \pt(\zv|\xv)] 
\end{array}
\end{align*}
For this model $\alpha_1 = -1$, $\alpha_2 = -1$, so mutual information is maximized. 



\section{Formal Closure Theorem and Proof}





We first consider the scenario where all divergences are $\KL$ divergence or reverse $\KL$ divergence, and then generalize to all divergences. Formally we consider a subset of possible Lagrangian dual functions of Eq.(\ref{eq:lag_dual}) which we define below
\begin{definition}[KL Lagrangian Objective]
\label{def:kl_lagrangian}
A KL Lagrangian objective $\mc{L}: \Theta \to \mathbb{R}$ is a linear combination of the following expressions: 
\begin{align*}
I_{\qt}(\xv; \zv) & \quad I_{\pt}(\xv; \zv) \\
\KL(\pt(\xv, \zv)||\qt(\xv, \zv)) & \quad \KL(\qt(\xv, \zv)||\pt(\xv, \zv)) \\
\KL(\pt(\xv)||q(\xv)) & \quad \KL(q(\xv)||\pt(\xv)) \\ 
\KL(p(\zv)||\qt(\zv)) & \quad \KL(\qt(\zv)||p(\zv)) \\ 
\Eb_{\pt(\xv)} \KL(\pt(\zv|\xv)||\qt(\zv|\xv)) & \quad \Eb_{q(\xv)} \KL(\qt(\zv|\xv)||\pt(\zv|\xv)) \\
\Eb_{p(\zv)} \KL(\pt(\xv|\zv)||\qt(\xv|\zv)) & \quad \Eb_{\qt(\zv)} \KL(\qt(\xv|\zv)||\pt(\xv|\zv)) \numberthis \label{eq:kl_terms}
\end{align*}
\end{definition}





We also provide a more formal version of Definition~\ref{def:equivalence}.

\textbf{Definition 1} (formal) Define the following set of equivalences as the elementary equivalences
\begin{align*}
\E_*[\log \pt(\xv, \zv)] &\equiv \E_*[\log \pt(\xv) + \log \pt(\zv|\xv)] \\
\E_*[\log \pt(\xv, \zv)] &\equiv \E_*[\log p(\zv) + \log \pt(\xv|\zv)] \\
\E_*[\log \qt(\xv, \zv)] &\equiv \E_*[\log q(\xv) + \log \qt(\zv|\xv)] \\
\E_*[\log \qt(\xv, \zv)] &\equiv \E_*[\log \qt(\zv) + \log \qt(\xv|\zv)] \\
\E_{q(x)}[\log q(\xv)] &\equiv 0 \,\qquad \E_{p(z)}[\log p(\zv)] \equiv 0
\end{align*}
where each $\E_*$ indicates two equivalences: one for $\E_{\pt(\xv, \zv)}$ and one for $\E_{\qt(\xv, \zv)}$. For each equivalence $\mathrm{LHS} \equiv \mathrm{RHS}$ defined above, we refer to $\mathrm{LHS} - \mathrm{RHS}$ as an elementary null expression. An optimization objective $\mc{L}$ is elementary equivalent to another optimization objective $\mc{L}'$ if $\mc{L}' - \mc{L}$ is a linear combination of elementary null expressions. 


Based on these definition we can formalize Theorem 1. We will separately consider the cases of likelihood based (T1), unary likelihood free (T2) and binary likelihood free (T3) objectives as defined in Definition~\ref{def:tractable_families}. 

\begin{theorem}
\label{thm:elementary_equivalence1}
Define the following KL Lagrangian objectives 

\textbf{1)} Variational mutual information maximization 
\begin{align*}
\mathcal{L}_{\mathrm{VMI}} = -I_{\pt}(\xv; \zv) + \KL(\pt(\zv|\xv)\Vert \qt(\zv|\xv)) = \E_{\pt(\xv, \zv)}[\log \qt(\zv|\xv)]
\end{align*}
\textbf{2)} $\beta$-VAE, where $\beta$ is any real number
\begin{align*}
\mathcal{L}_{\beta-\mathrm{VAE}} &= (\beta - 1) I_{\qt}(\xv, \zv) + \KL(\qt(\xv|\zv)\Vert \pt(\xv|\zv)) + \beta \KL(\qt(\zv)\Vert p(\zv)) \\
&= - \E_{\qt(\xv, \zv)}[\log \pt(\xv|\zv)] + \beta \E_{\qt(\xv, \zv)}[\log \qt(\zv|\xv)] - \beta \E_{\qt(\zv)}[\log p(\zv)] 
\end{align*}
Then any likelihood based computable KL Lagrangian objective is elementary equivalent to a linear combination of $\mathcal{L}_{\beta-\mathrm{VAE}}$ (for some $\beta \in \mathbb{R}$) and $\mathcal{L}_{\mathrm{VMI}}$. 
\end{theorem}


\begin{theorem}
\label{thm:elementary_equivalence2}
Define the following KL Lagrangian objectives

\textbf{3)} KL InfoGAN, where $\lambda_1, \lambda_2$ are any real number
\begin{align*}
\mathcal{L}_{\mathrm{InfoGAN}} &= - I_{\pt}(\xv; \zv) + \KL(\pt(\zv|\xv)\Vert \qt(\zv|\xv)) + \lambda_1 \KL(\pt(\xv)\Vert q(\xv)) + \lambda_2 \KL(q(\xv)\Vert \pt(\xv)) \\
&= - \E_{\pt(\xv, \zv)}[\log \qt(\zv|\xv)] + \lambda_1 \KL(\pt(\xv)\Vert q(\xv)) + \lambda_2 \KL(q(\xv)\Vert \pt(\xv))
\end{align*}
\textbf{4)} KL InfoVAE, where $\alpha_1, \lambda_3, \lambda_4$ are any real number
\begin{align*}
\mathcal{L}_{\mathrm{InfoVAE}} &= \alpha_1 I_{\qt}(\xv; \zv) + \KL(\qt(\xv|\zv)\Vert \pt(\xv|\zv)) + \lambda_3 \KL(\qt(\zv)\Vert p(\zv)) + \lambda_4 \KL(p(\zv)\Vert \qt(\zv)) \\
&= - \E_{\qt(\xv, \zv)}[\log \pt(\xv|\zv)] +  (\alpha + 1) \E_{\qt(\xv, \zv)}[\log \qt(\zv|\xv)] - (\alpha + 1) \E_{\qt(\zv)}[\log p(\zv)] + \\
& \qquad (\lambda_3 - \alpha_1 - 1) \KL(\qt(\zv)\Vert p(\zv)) + \lambda_4 \KL(p(\zv)\Vert \qt(\zv))
\end{align*}
Then any unary likelihood free computable KL Lagrangian objective is elementary equivalent to a linear combination of $\mathcal{L}_{\mathrm{InfoGAN}}$ (for some $\lambda_1, \lambda_2 \in \mathbb{R}$) and $\mathcal{L}_{\mathrm{InfoVAE}}$ (for some $\alpha_1, \lambda_3, \lambda_4 \in \mathbb{R}$).
\end{theorem}

\begin{theorem}
\label{thm:elementary_equivalence3}
Define the following KL Lagrangian objective

\textbf{5)} KL InfoBiGAN, where $\alpha_2, \lambda_5, \lambda_6$ are any real number 
\begin{align*}
\mathcal{L}_{\mathrm{InfoBiGAN}} &= \alpha_2 I_{\pt}(\xv; \zv) + \KL(\pt(\zv|\xv)\Vert \qt(\zv|\xv)) + \lambda_5 \KL(\pt(\xv)\Vert q(\xv)) + \lambda_6 \KL(q(\xv)\Vert \pt(\xv)) \\
&= (\alpha_2 + 1) \KL(\pt(\xv, \zv)\Vert \qt(\xv, \zv)) + (\lambda_5 - \alpha_2 - 1) \KL(\pt(\xv)\Vert q(\xv)) + \\
&\qquad \lambda_6 \KL(q(\xv)\Vert \pt(\xv)) - \alpha_2 \E_{\pt(\xv, \zv)}[\log \qt(\zv|\xv)]
\end{align*}
Then any binary likelihood free computable KL Lagrangian objective of the  Objective family is elementary equivalent to a linear combination of $\mathcal{L}_{\mathrm{InfoVAE}}$ (for some $\alpha_1, \lambda_3, \lambda_4 \in \mathbb{R}$) and $\mathcal{L}_{\mathrm{InfoBiGAN}}$ (for some $\alpha_2, \lambda_5, \lambda_6 \in \mathbb{R}$).
\end{theorem}

We briefly sketch the proof before a formal proof.  

\textbf{1)} We find a set of ``basic terms'' (e.g. $\Eb_{q(\xv)}[\log \pt(\xv)]$, we will formally define in theorem) as a  ``basis'' $\Ev = (\text{term}_1, \text{term}_2, \cdots)$. This basis spans a vector space: for each vector $\mathbf{n} \in \mathbb{R}^{|\Ev|}$, $\Ev \mathbf{n}$ is a linear combination of these basic terms. The set of terms in $\Ev$ must be large enough so that step \textbf{2)} is possible.

\textbf{2)} Each KL Lagrangian objective $\mc{L}$ in Definition~\ref{def:kl_lagrangian} is a linear combination of basic terms (columns of $\Ev$), so can be represented a a real vector $\mathbf{n}$ under basis $\Ev$: $\mc{L} = \Ev \mathbf{n}$. Similarly, each tractable expression in Definition~\ref{def:tractable_families}, and each elementary null expression in Definition~\ref{def:equivalence}, is a linear combination of basic terms. Both can be represented as real vectors under basis $\Ev$. 

\textbf{3)} By Definition~\ref{def:tractable_families}, a KL Lagrangian objective is (likelihood based/unary likelihood free/binary likelihood free) computable, if and only if we can subtract from this objective some \emph{elementary null expressions} (See Definition~\ref{def:equivalence}) to derive a linear combination of (likelihood based/unary likelihood free/binary likelihood free) computable expressions. 
This relationship can be interpreted as a subspace space equivalence relationship. If the the following two subspaces are identical: 

1. Direct sum of the subspace spanned by a set of KL Lagrangian objectives and the subspace spanned by elementary null expressions

2. Direct sum of the subspace spanned by (likelihood based/unary likelihood free/binary likelihood free) expressions and the subspace spanned by elementary null expressions

we can conclude that this set of KL Lagrangian objectives include every objective in that computation hardness class. This is exactly the statement in in Theorem~\ref{thm:elementary_equivalence1} \ref{thm:elementary_equivalence2} \ref{thm:elementary_equivalence3}.

\begin{proof}[Proof of Theorem~\ref{thm:elementary_equivalence1} \ref{thm:elementary_equivalence2} \ref{thm:elementary_equivalence3}]

We map all optimization objectives into a vector space $V \subseteq \left(\Theta \rightarrow \mathbb{R}\right)$ with standard function addition and scalar multiplication as operations. 
Define the following set of log probability measures:
\begin{align*}
\pv_{\atoms} = \left( \begin{array}{c} \log \pt(\xv, \zv) \\ \log \pt(\xv|\zv) \\ \log \pt(\zv|\xv) \\ \log \pt(\xv) \\ \log p(\zv) \end{array} \right) \quad
\qv_{\atoms} = \left( \begin{array}{c} \log \qt(\xv, \zv) \\ \log \qt(\xv|\zv) \\ \log \qt(\zv|\xv) \\ \log q(\xv) \\ \log \qt(\zv) \end{array} \right) 
\end{align*}
Based on expectation over these probability measures, we define the following set of 20 basic terms as our basis
\begin{align*}
\Ev = \left( \begin{array}{c} \E_{\pt(\xv, \zv)}[\pv_{\atoms}] \\ \E_{\pt(\xv, \zv)}[\qv_{\atoms}] \\ \E_{\qt(\xv, \zv)}[\pv_{\atoms}] \\ \E_{\qt(\xv, \zv)}[\qv_{\atoms}] \end{array} \right)^T
\end{align*}
We proceed to show that all KL Lagrangian objectives defined in Definition~\ref{def:kl_lagrangian}, tractable expressions in Definition~\ref{def:tractable_families}, and elementary null expressions defined Definition~\ref{def:equivalence} are linear combinations of this set of basis terms.

\textbf{1) KL Lagrangian Objectives}. 
Let $R \in \mathbb{R}^{20 \times 12}$ be
\begin{align*}
R=
\left(\begin{array}{cc}
\begin{array}{cccccc}
0 & 1 & 0 & 0 & 0 & 0 \\
0 & 0 & 1 & 0 & 0 & 0 \\
1 & 0 & 0 & 1 & 0 & 0 \\
0 & 0 & 0 & 0 & 1 & 0 \\
-1 & 0 & 0 & 0 & 0 & 1 \\
0 & -1 & 0 & 0 & 0 & 0 \\
0 & 0 & -1 & 0 & 0 & 0 \\
0 & 0 & 0 & -1 & 0 & 0 \\
0 & 0 & 0 & 0 & -1 & 0 \\
0 & 0 & 0 & 0 & 0 & -1 \\
\end{array} & 0 \\
0 & 
\begin{array}{cccccc}
0 & 1 & 0 & 0 & 0 & 0 \\
0 & 0 & 1 & 0 & 0 & 0 \\
0 & 0 & 0 & 1 & 0 & 0 \\
0 & 0 & 0 & 0 & 1 & 0 \\
0 & 0 & 0 & 0 & 0 & 1 \\
0 & -1 & 0 & 0 & 0 & 0 \\
0 & 0 & -1 & 0 & 0 & 0 \\
1 & 0 & 0 & -1 & 0 & 0 \\
0 & 0 & 0 & 0 & -1 & 0 \\
-1 & 0 & 0 & 0 & 0 & -1 \\
\end{array} 
\end{array} \right)
\end{align*}
Each column of the matrix is one of the twelve expressions defined in Definition~\ref{def:kl_lagrangian} under the basis $\Ev$. 
Therefore any KL Lagrangian optimization objective $\mathcal{L}$ can be represented as a vector $\nv \in \mathbb{R}^{12}$, where 
\[ \mathcal{L} = \Ev R \nv \]

Denote the subspace spanned by the columns of $R$ as $\mathcal{R}$. This is a subspace in $\mathbb{R}^{20}$ and   represents all KL Lagrangian objectives under the basis $\Ev$.

\textbf{2) Elementary transformations}. Consider all the elementary null expressions in Definition~\ref{def:equivalence}. 
We define the matrix $P \in \mathbb{R}^{20 \times 10}$ as
\begin{align*}
P=\left(\begin{array}{cc}
\begin{array}{cc}
\begin{array}{ccc}
1 & 1 & 0 \\
-1 & 0 & 0 \\
0 & -1 & 0 \\
0 & -1 & 0 \\ 
-1 & 0 & 1 
\end{array}  &  \\
 &
\begin{array}{cc}
1 & 1  \\
-1 & 0  \\
0 & -1  \\
0 & -1  \\ 
-1 & 0  
\end{array} \end{array}
& 0 \\
0 & 
\begin{array}{cc}
\begin{array}{cc}
1 & 1  \\
-1 & 0 \\
0 & -1 \\
0 & -1 \\ 
-1 & 0  
\end{array} &  \\
 &
\begin{array}{ccc}
1 & 1 & 0 \\
-1 & 0 & 0 \\
0 & -1 & 0  \\
0 & -1 & 1  \\ 
-1 & 0 & 0 
\end{array} \end{array}
\end{array} 
\right)
\end{align*}
For an objective $\mc{L} = \Ev R \nv$, and for each column $\cv$ of P, adding $\Ev R \nv + \Ev \cv$ corresponds to adding an elementary null expression. Because each elementary null expression is a constant with respect to the trainable parameters $\theta$, this means that $\Ev R \nv + \Ev \cv \equiv \Ev R \nv$. In addition, for any  vector $\bv \in \mathbb{R}^{10}$, $\Ev R \nv + \Ev P \bv \equiv \Ev R \nv$. This corresponds to adding multiple elementary null expressions. Denote the subspace spanned by the columns $P$ as $\mc{P}$. 





\textbf{3) Hardness Classes}. 
Expressions with different hardness are given by Definition~\ref{def:tractable_families}. As before, we find the corresponding matrix $T$ and subspace $\mathcal{T}$ for each ``hardness class''.  



For likelihood based computable objectives, the eight likelihood based computable terms in Definition~\ref{def:tractable_families} can be written as the columns of $T_{\mathrm{lb}} \in \mathbb{R}^{20 \times 8}$ under basis $\Ev$:
\begin{align*}
T^p_{\mathrm{lb}} = T^q_{\mathrm{lb}} = 
\left( \begin{array}{cccc}
1 & 0 & 0 & 0 \\
0 & 1 & 0 & 0 \\
0 & 0 & 0 & 0 \\
0 & 0 & 0 & 0 \\
0 & 0 & 1 & 0 \\
0 & 0 & 0 & 0 \\
0 & 0 & 0 & 0 \\
0 & 0 & 0 & 1 \\
0 & 0 & 0 & 0 \\
0 & 0 & 0 & 0 \\
\end{array} \right) 
\quad 
T_{\mathrm{lb}} = \left( \begin{array}{cc}
T_{\mathrm{lb}}^p & 0 \\ 
0 & T_{\mathrm{lb}}^q \\
\end{array} \right)
\end{align*}
denote the corresponding subspace as $\mathcal{T}_{\mathrm{lb}}$. 






For unary likelihood free computable objectives, aside from all likelihood based computable objectives, we also add the four divergences 
\[ \KL(\pt(\xv)||q(\xv)), \KL(q(\xv)||\pt(\xv)), \KL(p(\zv)||\qt(\zv)), \KL(\qt(\zv)||p(\zv)) \]
similarly we have 
\begin{align*}
T_{\mathrm{ulf}}^p = T_{\mathrm{ulf}}^q = 
\left( \begin{array}{cc} T_{\mathrm{lb}}^p & \begin{array}{cc}
0 & 0 \\
0 & 0 \\
0 & 0 \\
-1 & 0 \\
0 & -1 \\
0 & 0 \\
0 & 0 \\
0 & 0 \\
1 & 0 \\
0 & 1 
\end{array} \end{array} \right)
\quad
T_{\mathrm{ulf}} = \left( \begin{array}{cc}
T_{\mathrm{ulf}}^p & 0 \\ 
0 & T_{\mathrm{ulf}}^q \\
\end{array} \right)
\end{align*}
All unary likelihood free computable objectives are columns of $T_{\mathrm{ulf}} \in \mathbb{R}^{20 \times 12}$ under basis $\Ev$, and denote the corresponding subspace as $\mathcal{T}_{\mathrm{ulf}}$.

For binary likelihood free computable objectives, aside from all unary likelihood free computable objectives, we also add the two divergences
\[ \KL(\pt(\xv, \zv)||\qt(\xv, \zv)), \KL(\qt(\xv, \zv)||\pt(\xv, \zv)) \]
to get 
\begin{align*}
T_{\mathrm{blf}}^p = T_{\mathrm{blf}}^q = 
\left( \begin{array}{cc} T_{\mathrm{ulf}}^p & \begin{array}{cc}
1 \\
0 \\
0 \\
0 \\
0  \\
-1 \\
0 \\
0 \\
0 \\
0 
\end{array} \end{array} \right)
\quad
T_{\mathrm{blf}} = \left( \begin{array}{cc}
T_{\mathrm{blf}}^p & 0 \\ 
0 & T_{\mathrm{blf}}^q \\
\end{array} \right)
\end{align*}
All binary likelihood free computable objectives are columns of $T_{\mathrm{blf}} \in \mathbb{R}^{20 \times 14}$ under basis $\Ev$, denote the corresponding subspace as $\mathcal{T}_{\mathrm{blf}}$.



\textbf{4) Known Objective Families}. 
We find the subspace spanned by linear combination of the two likelihood based computable objectives $\beta$-VAE and Variational mutual information maximization under basis $\Ev$. This can be represented as the column space of 
\begin{align*}
\left( \begin{array}{cc} S_{\mathrm{VMI}} & S_{\beta-\mathrm{VAE}} \end{array} \right) = 
\left(\begin{array}{cc}
\begin{array}{c}
0 \\
0 \\
0 \\
0 \\
1 \\
0 \\
0 \\
-1 \\
0 \\
0 \\
\end{array} & 0 \\
0 & 
\begin{array}{cc}
0 & 0 \\
-1 & 0 \\
0 & 0 \\
0 & 0 \\
0 & -1 \\
0 & 0 \\
0 & -1 \\
0 & 0 \\
1 & 1 \\
0 & 1 \\
\end{array}
\end{array} \right)
\end{align*}
Denote the corresponding column space as $\mathcal{S}_{\mathrm{VMI}} + \mathcal{S}_{\beta-\mathrm{VAE}}$. Note that the matrix contains three columns because VMI has no free parameter so spans a one dimensional subspace, while $\beta$-VAE has a free parameter $\beta$, so spans a two dimensional subspace.


The subspace spanned by linear combinations of InfoGAN and InfoVAE is the column space of:
\begin{align*}
\left( \begin{array}{cc} S_{\mathrm{InfoGAN}} & S_{\mathrm{InfoVAE}} \end{array} \right) = 
\left(\begin{array}{cccc}
\begin{array}{cc}
0 & 0 \\
-1 & 0 \\
1 & 0 \\
1 & 1 \\
0 & 0 \\
0 & 0 \\
0 & 0 \\
-1 & 0 \\
0 & -1 \\
0 & 0 \\
\end{array} & 0 & \begin{array}{c}
0 \\ 
0 \\ 
0 \\ 
-1 \\ 
0 \\ 
0 \\ 
0 \\ 
0 \\ 
1 \\ 
0 \\ 
\end{array} & 0 \\
0 & \begin{array}{c} 
0 \\ 
0 \\ 
0 \\ 
0 \\ 
1 \\ 
0 \\ 
0 \\ 
0 \\ 
0 \\ 
-1 \\ 
\end{array} & 0 & \begin{array}{ccc}
0 & 0 & 0 \\
0 & -1 & 0 \\
0 & 0 & 0 \\
0 & 0 & 0 \\
0 & 0 & -1 \\
0 & 0 & 0 \\
0 & 1 & 0 \\
1 & 0 & 0 \\
0 & 0 & 0 \\
-1 & 0 & 1 \\
\end{array}
\end{array} \right)
\end{align*}
Denote the corresponding column space as $\mathcal{S}_{\mathrm{InfoGAN}} + \mathcal{S}_{\mathrm{InfoVAE}}$.

For subspace spanned by linear combinations of InfoBiGAN and InfoVAE is the column space of:
\begin{align*}
\left( \begin{array}{cc} S_{\mathrm{InfoBiGAN}} & S_{\mathrm{InfoVAE}} \end{array} \right) = 
\left(\begin{array}{cccc}
\begin{array}{ccc}
0 & 0 & 0 \\
1 & 0 & 0 \\
0 & 1 & 0 \\
-1 & 0 & 1 \\
0 & 0 & 0 \\
0 & 0 & 0 \\
0 & 0 & 0 \\
0 & -1 & 0 \\
0 & 0 & -1 \\
0 & 0 & 0 \\
\end{array} & 0 & \begin{array}{c}
0 \\ 
0 \\ 
0 \\ 
-1 \\ 
0 \\ 
0 \\ 
0 \\ 
0 \\ 
1 \\ 
0 \\ 
\end{array} & 0 \\
0 & \begin{array}{c} 
0 \\ 
0 \\ 
0 \\ 
0 \\ 
1 \\ 
0 \\ 
0 \\ 
0 \\ 
0 \\ 
-1 \\ 
\end{array} & 0 & \begin{array}{ccc}
0 & 0 & 0 \\
0 & -1 & 0 \\
0 & 0 & 0 \\
0 & 0 & 0 \\
0 & 0 & -1 \\
0 & 0 & 0 \\
0 & 1 & 0 \\
1 & 0 & 0 \\
0 & 0 & 0 \\
-1 & 0 & 1 \\
\end{array}
\end{array} \right)
\end{align*}
Denote the corresponding column space as $\mathcal{S}_{\mathrm{InfoBiGAN}} + \mathcal{S}_{\mathrm{InfoVAE}}$.




\textbf{5) Subspace Equivalence Relationship}. 

As before denote $\mathcal{A} + \mathcal{B}$ as the sum of two subspaces $\mathcal{A}$ and $\mathcal{B}$ (i.e. $\mathcal{A} + \mathcal{B}$ is the set of all vectors that are linear combinations of vectors in $\mathcal{A}$ and vectors in $\mathcal{B}$.) 
Consider first the $\mathcal{L}_{\beta-\mathrm{VAE}}$ objective and $\mathcal{L}_{\mathrm{VMI}}$ objective.
\begin{align*}
\mathcal{L}_{\mathrm{VMI}} &= -I_{\pt}(\xv; \zv) + \KL(\pt(\zv|\xv)\Vert \qt(\zv|\xv)) \numberthis \label{eq:vmi_form1} \\
&= \E_{\pt(\xv, \zv)}[\log \qt(\zv|\xv)] \numberthis \label{eq:vmi_form2}
\end{align*}
\begin{align*}
\mathcal{L}_{\beta-\mathrm{VAE}} &= (\beta - 1) I_{\qt}(\xv, \zv) + \KL(\qt(\xv|\zv)\Vert \pt(\xv|\zv)) + \beta \KL(\qt(\zv)\Vert p(\zv)) \numberthis \label{eq:elbo_form1} \\
&= - \E_{\qt(\xv, \zv)}[\log \pt(\xv|\zv)] + \beta \E_{\qt(\xv, \zv)}[\log \qt(\zv|\xv)] - \beta \E_{\qt(\zv)}[\log p(\zv)] \numberthis \label{eq:elbo_form2} 
\end{align*}
The $\mathcal{L}_{\beta-\mathrm{VAE}}$ objective and the $\mathcal{L}_{\mathrm{VMI}}$ objective are Lagrangian dual forms (Eq.(\ref{eq:vmi_form1})(\ref{eq:elbo_form1})), but are also elementary equivalent to an expression that only contains terms in $\mathcal{T}_{lb}$ in Eq.(\ref{eq:vmi_form2})\ref{eq:elbo_form2}). 
This means that for any $\sv \in \mathcal{S}_{\beta-\mathrm{VAE}}$ or $\sv \in \mathcal{S}_{VMI}$, there is an $\rv$ in the space of Lagrangian dual forms $\mathcal{R}$, a $\tv \in \mathcal{T}_{\mathrm{lb}}$ in the space of likelihood based computable expressions, and a $\pv$ in space of elementary null expressions $\mathcal{P}$ so that $\sv = \rv + \pv$ and $\sv = \tv + \pv$. This implies 
\begin{align*}
\mathcal{S}_{\beta-\mathrm{VAE}} + \mathcal{S}_{VMI} &\subset \mathcal{R} + \mathcal{P} \\
\mathcal{S}_{\beta-\mathrm{VAE}} + \mathcal{S}_{VMI} &\subset \mathcal{T}_{\mathrm{lb}} + \mathcal{P} 
\end{align*}
Therefore
\begin{align*}
\mathcal{S}_{\beta-\mathrm{VAE}} + \mathcal{S}_{VMI} \subset (\mathcal{R} + \mathcal{P}) \cap (\mathcal{T}_{\mathrm{lb}} + \mathcal{P})
\end{align*}

In addition we have the trivial relationship
\[ \mathcal{P} \subset (\mathcal{R} + \mathcal{P}) \cap (\mathcal{T}_{\mathrm{lb}} + \mathcal{P}) \]
so
\begin{align*}
\mathcal{S}_{\beta-\mathrm{VAE}} + \mathcal{S}_{\mathrm{VMI}} + \mathcal{P} \subset (\mathcal{R} + \mathcal{P}) \cap (\mathcal{T}_{\mathrm{lb}} + \mathcal{P}) \numberthis \label{equ:dim_match1}
\end{align*}
Therefore if we can further have
\[ \mathrm{dim}(\mathcal{S}_{\beta-\mathrm{VAE}} + \mathcal{S}_{\mathrm{VMI}} + \mathcal{P}) = \mathrm{dim} \left( (\mathcal{R} + \mathcal{P}) \cap (\mathcal{T}_{\mathrm{lb}} + \mathcal{P}) \right) \]
the above subspaces must be identical. This is because --- by contradiction --- that if the two spaces are not equivalent even when the condition in Eq.(\ref{equ:dim_match1}) holds, there exists vector $\vv \in (\mathcal{R} + \mathcal{P}) \cap (\mathcal{T}_{\mathrm{lb}} + \mathcal{P})$ but $\vv \not\in \mathcal{S}_{\beta-\mathrm{VAE}} + \mathcal{S}_{\mathrm{VMI}} + \mathcal{P}$. 
This implies that
\begin{align*}
\mathrm{dim} & \left( \mathcal{S}_{\beta-\mathrm{VAE}} + \mathcal{S}_{\mathrm{VMI}} + \mathcal{P} + \lbrace \vv \rbrace \right) >  
\mathrm{dim}\left(\mathcal{S}_{\beta-\mathrm{VAE}} + \mathcal{S}_{\mathrm{VMI}} + \mathcal{P}\right) = \mathrm{dim} \left( (\mathcal{R} + \mathcal{P}) \cap (\mathcal{T}_{\mathrm{lb}} + \mathcal{P}) \right)
\end{align*}
However $\mathcal{S}_{\beta-\mathrm{VAE}} + \mathcal{S}_{\mathrm{VMI}} + \mathcal{P} + \lbrace \vv \rbrace$ is a subspace of $(\mathcal{R} + \mathcal{P}) \cap (\mathcal{T}_{\mathrm{lb}} + \mathcal{P})$ but has greater dimension, leading to a contradiction. 

Similarly we have for the other families
\begin{align*}
\mathcal{S}_{\mathrm{InfoGAN}} + \mathcal{S}_{\mathrm{InfoVAE}} + \mathcal{P} &\subset (\mathcal{R} + \mathcal{P}) \cap (\mathcal{T}_{\mathrm{ulf}} + \mathcal{P}) \numberthis \label{equ:dim_match2} \\
\mathcal{S}_{\mathrm{InfoBiGAN}} + \mathcal{S}_{\mathrm{InfoVAE}} + \mathcal{P} &\subset (\mathcal{R} + \mathcal{P}) \cap (\mathcal{T}_{blf} + \mathcal{P}) \numberthis \label{equ:dim_match3}
\end{align*}

To compute the above subspace dimensions, we find a basis matrix whose columns spans the subspace. The rank of the basis matrix is equal to the dimensionality of its column space. 

We have written down the basis matrix for $\mc{P}$, $\mc{R}$, $\mc{T}_{\mathrm{lb}}$, $\mc{T}_{\mathrm{ulf}}$, $\mc{T}_{\mathrm{blf}}$, $\mathcal{S}_{\beta-\mathrm{VAE}}$, $\mathcal{S}_{\mathrm{VMI}}$, $\mathcal{S}_{\mathrm{InfoGAN}}$, $\mathcal{S}_{\mathrm{InfoVAE}}$, $\mathcal{S}_{\mathrm{InfoBiGAN}}$. We need to also find the basis matrix for the union and intersection of these subspaces, which we now derive. 

Let $\mathcal{A}, \mathcal{B}$ be the column space of matrices $A, B$, then the basis matrix of $\mathcal{A} + \mathcal{B}$ can be derived by $\left( \begin{array}{cc} A & B \end{array} \right)$. The basis matrix of $\mathcal{A} \cap \mathcal{B}$ can be derived by the following process: Compute the solution space for the linear equality 
\[ \left( \begin{array}{cc} A & B \end{array} \right) \left(\begin{array}{c} \uv \\ \vv \end{array} \right) = 0 \]
Let $U, V$ be basis matrix for the  solution space of $\uv, \vv$ respectively. Then $AU$ (or equivalently $-BV$) is a basis matrix for $\mathcal{A} \cap \mathcal{B}$~\citep{strang1993introduction}.

We can perform this procedure for all the above mentioned subspaces, and derive the following table. 


\begin{center}
\begin{tabular}{c|c}
subspace  & dimension \\
\hline 
$ (\mathcal{S}_{\beta-\mathrm{VAE}} + \mathcal{S}_{\mathrm{VMI}}) + \mathcal{P} $ & 13 \\
$  (\mathcal{T}_{lb} + \mathcal{P}) \cap (\mathcal{R} + \mathcal{P}) $ & 13 \\
\hline 
$ (\mathcal{S}_{\mathrm{InfoGAN}} \cup \mathcal{S}_{\mathrm{InfoVAE}}) + \mathcal{P} $ & 17 \\ 
$ (\mathcal{T}_{ulf} + \mathcal{P}) \cap (\mathcal{R} + \mathcal{P})$ & 17 \\ 
\hline 
$ (\mathcal{S}_{\mathrm{InfoBiGAN}} \cup \mathcal{S}_{\mathrm{InfoVAE}}) + \mathcal{P} $ & 18 \\ 
$ (\mathcal{T}_{blf} + \mathcal{P}) \cap (\mathcal{R} + \mathcal{P}) $ & 18
\end{tabular}
\end{center}

Therefore we have verified that dimensionality match for Eq. (\ref{equ:dim_match1}, \ref{equ:dim_match2}, \ref{equ:dim_match3}). Therefore we have
\begin{align*}
\mathcal{S}_{\beta-\mathrm{VAE}} + \mathcal{S}_{\mathrm{VMI}} + \mathcal{P} &= (\mathcal{R} + \mathcal{P}) \cap (\mathcal{T}_{\mathrm{lb}} + \mathcal{P}) \\
\mathcal{S}_{\mathrm{InfoGAN}} + \mathcal{S}_{\mathrm{InfoVAE}} + \mathcal{P} &= (\mathcal{R} + \mathcal{P}) \cap (\mathcal{T}_{\mathrm{ulf}} + \mathcal{P})  \\
\mathcal{S}_{\mathrm{InfoBiGAN}} + \mathcal{S}_{\mathrm{InfoVAE}} + \mathcal{P} &= (\mathcal{R} + \mathcal{P}) \cap (\mathcal{T}_{blf} + \mathcal{P}) 
\end{align*}

This implies that for any vector $\rv \in \mathcal{R}$ (which corresponds to objective $\Ev \rv$), if $\rv$ can be converted into likelihood based tractable form by elementary transformations: $\rv \in \mathcal{T}_{\mathrm{lb}}+\mathcal{P}$, then 
\[ \rv \in \mathcal{R} \cap (\mathcal{T}_{\mathrm{lb}}+\mathcal{P}) \subset  (\mathcal{R} + \mathcal{P}) \cap (\mathcal{T}_{\mathrm{lb}} + \mathcal{P}) \]
This means 
\[ \rv \in \mathcal{S}_{\beta-\mathrm{VAE}} + \mathcal{S}_{\mathrm{VMI}} + \mathcal{P} \]
which implies that $\rv$ can be converted by elementary transformation into a linear combination of vectors in $\mathcal{S}_{\beta-\mathrm{VAE}}$ and vectors in $\mathcal{S}_{VMI}$. We can derive identical conclusion for the other two hardness classes.  
\end{proof}

Finally we argue that the closure result holds for other divergences. Because all the elementary null expressions defined in Definition~\ref{def:equivalence} involve adding or subtracting terms of the form $E_*[\log \cdot]$ where $*$ is $p_\theta(\xv, \zv)$ or $q_\theta(\xv, \zv)$ and $\cdot$ is any joint, marginal or conditional distribution of $p_\theta$ or $q_\theta$. Note that out of the divergences we discussed in Section 2: Wasserstein distance, f-divergence, MMD distance, JS divergence, no divergence other than the KL divergence/reverse KL divergence can add a elementary null expression, and transform into a difference expression in Definition~\ref{def:tractable_families}. Therefore such divergences must be respectively optimized and there can be no computational shortcut by elementary transformations. 

\section{Proofs for Section 5}
\textbf{Lemma \ref{lemma:convex_constraint}}[Convex Constraints (Formal)]
For any $D$ that is $\KL$ $\MMD$, or $D_f$, the following expressions are convex with respect to $\theta \in \Theta$ as defined in Eq.(\ref{eq:theta_set}):
\begin{align*}
&D(p_\theta(\xv, \zv) \Vert q_\theta(\xv, \zv))  & D(q_\theta(\xv, \zv) \Vert p_\theta(\xv, \zv)) \\
&D(p_\theta(\xv) \Vert q(\xv)) & D(q(\xv) \Vert p_\theta(\xv))  \\
&D(p(\zv) \Vert q_\theta(\zv)) & D(q_\theta(\zv) \Vert p(\zv)) 
\end{align*}

\begin{proof}[Proof of Lemma~\ref{lemma:convex_constraint}]
We first prove convexity for $\KL(p(\zv) \Vert q_\theta(\zv))$. For any probability measures on $\mathcal{Z}$ (or any discrete sample space), $\KL$ is jointly convex with respect to both arguments~\citep{liese2006divergences}: for any probability measures $p_1, p_2, q_1, q_2$ on $\mc{Z}$ we have
\[ \KL(\lambda p_1 + (1 - \lambda) p_2 \Vert \lambda q_1 + (1 - \lambda) q_2) \leq \lambda \KL(p_1 \Vert q_1) + (1 - \lambda) \KL(p_2 \Vert q_2),  \]
Therefore fixing $p_1=p_2=p$ we have 
\[ \KL(p \Vert \lambda q_1 + (1 - \lambda) q_2) \leq \lambda \KL(p \Vert q_1) + (1 - \lambda) \KL(p \Vert q_2)  \]
which implies that $\KL(p(\zv) \Vert q_\theta(\zv))$ is convex with respect to $q_\theta(\zv)$ in probability measure space (i.e. $q_\theta(\zv)$ is any vector of probabilities $[q_\theta(\zv=i), i \in \mathcal{Z}]$). 
In addition $q_\theta(\zv=i) = \sum_j q(\xv=j) \theta^q_{ij}$ is linear in $\theta^q$. This means that $\KL(p(\zv) \Vert q_\theta(\zv))$  is a convex function composed with a linear function of $\theta^q$, so must be convex with respect to $\theta^q$, and additionally must be convex with respect to $\theta$. 

For $\KL(p_\theta(\xv, \zv) \Vert q_\theta(\xv, \zv))$, as before $\KL$ is jointly convex with respect to $p_\theta(\xv, \zv)$ and $q_\theta(\xv, \zv)$ in probability measure space ($p_\theta(\xv, \zv)$, $q_\theta(\xv, \zv)$ are any vectors of probabilities $[p_\theta(\xv=i, \zv=j), q_\theta(\xv=k, \zv=l), i, k \in \mathcal{X}, j, l \in \mathcal{Z}]$). In addition, $p(\xv=i, \zv=j) = p(\zv=j) \theta^p_{ij}$ and $q(\xv=i, \zv=j) = q(\xv=i) \theta^q_{ij}$; both are linear in $\theta$. Therefore $\KL(p_\theta(\xv, \zv) \Vert q_\theta(\xv, \zv))$ is a convex function composed with a linear function of $\theta$, so must be convex with respect to $\theta$. 

The proof similarly generalizes to other cases. In addition, the proofs generalizes to different divergences $\KL$, $\MMD$, $\D_f$ because they are all jointly convex with respect to both arguments. This covers all scenarios in the lemma and completes the proof. 
\end{proof}

\begin{proof}[Proof of Lemma~\ref{lemma:convex_bound}]
Upper bound: Because $\overline{I}_{q_\theta} = \bb{E}_{q(\xv)}[\KL(q_\theta(\zv|\xv)\Vert p(\zv))] $, and $\KL(q_\theta(\zv|\xv)\Vert p(\zv))$ is convex in the first argument, and expectation of a convex function is a convex function. Therefore $\overline{I}_{q_\theta}$ is convex with respect to $\theta$. 

Lower bound: Because
\[ \underline{I}_{q_\theta} \equiv \Eb_{q_\theta(\xv, \zv)}[\log p_\theta(\xv|\zv)] \] is linear in $q_\theta(\xv, \zv)$, and $q_\theta(\xv=i, \zv=j) = q(\xv=i) \theta_{ij} $ is linear in $\theta_{ij}$. Therefore $\underline{I}_{q_\theta}$ is linear (concave) in $\theta$. 
\end{proof}

\begin{proof}[Proof for Lemma~\ref{lemma:non_empty_interior}]
Because $\Theta$ include all distributions on $(\mc{X}, \mc{Z})$, there exists a $\theta$ such that $p_\theta(\xv, \zv) = q_\theta(\xv, \zv)$. This can be achieved by setting $\forall \zv \in \mc{Z}, p_\theta(\xv|\zv) = q(\xv)$, and $\forall \xv \in \mc{X}, q_\theta(\zv|\xv) = p(\zv)$. Therefore there exists a solution that strictly satisfies the constraints $\mathcal{D} = 0 < \epsilonv$. 
\end{proof}

\section{Explicit Lagrangian VAE Objectives}
\label{sec:lave-dual}
For mutual information minimization / maximization, we can write the corresponding Lagrangian dual forms $\overline{h}$ and $\underline{h}$ for upper bounds and lower bounds. For $\alpha_1 < 0$ we consider
\begin{align}
& \overline{h}(\theta, \lambda; \alpha_1, \epsilonv) =  -\alpha_1   \overline{I}_q(\xv; \zv) + \lambdav^\top \mc{D} \nonumber \\
\equiv  & \ -\lambda_1 \epsilon_1^{-1} \bb{E}_{\qt(\xv, \zv)}[\log \pt(\xv|\zv)] - \lambda_1 - \lambda_2 \nonumber \\
&  \quad + (\lambda_1 \epsilon_1^{-1} - \alpha_1) \ \bb{E}_{q(\xv)}[\KL(\qt(\zv|\xv) \Vert p(\zv))] \nonumber \\
&  \quad + \lambda_2 \epsilon_2^{-1} \ \MMD(\qt(\zv) \Vert p(\zv)) \nonumber
\end{align}
For $\alpha_1 > 0$, we consider
\begin{align}
& \underline{h}(\theta, \lambda; \alpha_1, \epsilonv) = -\alpha_1  \underline{I}_q(\xv; \zv) + \lambdav^\top \mc{D} \nonumber \\
\equiv  & \ -(\lambda_1 \epsilon_1^{-1} + \alpha_1) \bb{E}_{\qt(\xv, \zv)}[\log \pt(\xv|\zv)] - \lambda_1 - \lambda_2 \nonumber \\
&  \ + \lambda_1 \epsilon_1^{-1} \ \bb{E}_{q(\xv)}[\KL(\qt(\zv|\xv) \Vert p(\zv))] \nonumber \\
&  \ + \lambda_2 \epsilon_2^{-1} \ \MMD(\qt(\zv) \Vert p(\zv))   \nonumber
\end{align}

In both objectives, the coefficient for likelihood is negative while the coefficients for the constraints are non-negative for all $\lambda_1, \lambda_2 \geq 0$; this ensures that we will not maximize divergence or minimize log-likelihood at any time, causing instability in optimization. 

\section{Experiment Details for Section~\ref{sec:exp-pareto}}
We consider both decoder and encoder networks to be 2 fully connected layers of 1024 neurons with softplus activations. For each set of hyperparameters, we can then obtain its mutual information and consistency values. 
We use the Adam optimizer~\cite{adam_optimization2014} with $\beta_1 = 0.5$. We empirically observe that setting $\beta_1 = 0.5$ to improve stability in adversarial training over $\theta$ and $\lambdav$, which has been suggested in literature about adversarial training~\citep{deconvolutional_gan2015}.

%% file: main.bbl
\begin{thebibliography}{39}
\providecommand{\natexlab}[1]{#1}
\providecommand{\url}[1]{\texttt{#1}}
\expandafter\ifx\csname urlstyle\endcsname\relax
  \providecommand{\doi}[1]{doi: #1}\else
  \providecommand{\doi}{doi: \begingroup \urlstyle{rm}\Url}\fi

\bibitem[Alemi et~al.(2017)Alemi, Poole, Fischer, Dillon, Saurous, and
  Murphy]{alemi2017fixing}
Alemi, Alexander~A., Poole, Ben, Fischer, Ian, Dillon, Joshua~V., Saurous,
  Rif~A., and Murphy, Kevin.
\newblock An information-theoretic analysis of deep latent-variable models.
\newblock \emph{CoRR}, abs/1711.00464, 2017.
\newblock URL \url{http://arxiv.org/abs/1711.00464}.

\bibitem[{Arjovsky} et~al.(2017){Arjovsky}, {Chintala}, and {Bottou}]{wgan2017}
{Arjovsky}, M., {Chintala}, S., and {Bottou}, L.
\newblock {Wasserstein GAN}.
\newblock \emph{ArXiv e-prints}, January 2017.

\bibitem[Barber \& Agakov(2003)Barber and Agakov]{barber2003algorithm}
Barber, David and Agakov, Felix.
\newblock The im algorithm: a variational approach to information maximization.
\newblock In \emph{Proceedings of the 16th International Conference on Neural
  Information Processing Systems}, pp.\  201--208. MIT Press, 2003.

\bibitem[Boyd \& Vandenberghe(2004)Boyd and Vandenberghe]{boyd2004convex}
Boyd, Stephen and Vandenberghe, Lieven.
\newblock \emph{Convex optimization}.
\newblock Cambridge university press, 2004.

\bibitem[Chen et~al.(2016{\natexlab{a}})Chen, Duan, Houthooft, Schulman,
  Sutskever, and Abbeel]{infogan2016}
Chen, Xi, Duan, Yan, Houthooft, Rein, Schulman, John, Sutskever, Ilya, and
  Abbeel, Pieter.
\newblock Infogan: Interpretable representation learning by information
  maximizing generative adversarial nets.
\newblock \emph{arXiv preprint arXiv:1606.03657}, 2016{\natexlab{a}}.

\bibitem[Chen et~al.(2016{\natexlab{b}})Chen, Kingma, Salimans, Duan, Dhariwal,
  Schulman, Sutskever, and Abbeel]{lossy_vae2016}
Chen, Xi, Kingma, Diederik~P, Salimans, Tim, Duan, Yan, Dhariwal, Prafulla,
  Schulman, John, Sutskever, Ilya, and Abbeel, Pieter.
\newblock Variational lossy autoencoder.
\newblock \emph{arXiv preprint arXiv:1611.02731}, 2016{\natexlab{b}}.

\bibitem[Dhar et~al.(2018)Dhar, Grover, and Ermon]{dhar2018sparse}
Dhar, Manik, Grover, Aditya, and Ermon, Stefano.
\newblock Sparse-gen: Modeling sparse deviations for compressed sensing using
  generative models.
\newblock \emph{International Conference on Machine Learning}, 2018.

\bibitem[Donahue et~al.(2016)Donahue, Kr{\"{a}}henb{\"{u}}hl, and
  Darrell]{bigan2016}
Donahue, Jeff, Kr{\"{a}}henb{\"{u}}hl, Philipp, and Darrell, Trevor.
\newblock Adversarial feature learning.
\newblock \emph{CoRR}, abs/1605.09782, 2016.
\newblock URL \url{http://arxiv.org/abs/1605.09782}.

\bibitem[Dumoulin et~al.(2016{\natexlab{a}})Dumoulin, Belghazi, Poole, Lamb,
  Arjovsky, Mastropietro, and Courville]{adversarially_learned_inference2016}
Dumoulin, Vincent, Belghazi, Ishmael, Poole, Ben, Lamb, Alex, Arjovsky, Martin,
  Mastropietro, Olivier, and Courville, Aaron.
\newblock Adversarially learned inference.
\newblock \emph{arXiv preprint arXiv:1606.00704}, 2016{\natexlab{a}}.

\bibitem[Dumoulin et~al.(2016{\natexlab{b}})Dumoulin, Belghazi, Poole, Lamb,
  Arjovsky, Mastropietro, and Courville]{ali2016}
Dumoulin, Vincent, Belghazi, Ishmael, Poole, Ben, Lamb, Alex, Arjovsky, Martin,
  Mastropietro, Olivier, and Courville, Aaron.
\newblock Adversarially learned inference.
\newblock \emph{arXiv preprint arXiv:1606.00704}, 2016{\natexlab{b}}.

\bibitem[Goodfellow et~al.(2014)Goodfellow, Pouget-Abadie, Mirza, Xu,
  Warde-Farley, Ozair, Courville, and Bengio]{generative_adversarial_nets2014}
Goodfellow, Ian, Pouget-Abadie, Jean, Mirza, Mehdi, Xu, Bing, Warde-Farley,
  David, Ozair, Sherjil, Courville, Aaron, and Bengio, Yoshua.
\newblock Generative adversarial nets.
\newblock In \emph{Advances in Neural Information Processing Systems}, pp.\
  2672--2680, 2014.

\bibitem[Gretton et~al.(2007)Gretton, Borgwardt, Rasch, Sch{\"o}lkopf, and
  Smola]{mmd_statistics2007}
Gretton, Arthur, Borgwardt, Karsten~M, Rasch, Malte, Sch{\"o}lkopf, Bernhard,
  and Smola, Alex~J.
\newblock A kernel method for the two-sample-problem.
\newblock In \emph{Advances in neural information processing systems}, pp.\
  513--520, 2007.

\bibitem[Grover et~al.(2018)Grover, Dhar, and Ermon]{grover2018flow}
Grover, Aditya, Dhar, Manik, and Ermon, Stefano.
\newblock Flow-{GAN}: Combining maximum likelihood and adversarial learning in
  generative models.
\newblock In \emph{AAAI Conference on Artificial Intelligence}, 2018.

\bibitem[Higgins et~al.(2016)Higgins, Matthey, Pal, Burgess, Glorot, Botvinick,
  Mohamed, and Lerchner]{beta_vae2016}
Higgins, Irina, Matthey, Loic, Pal, Arka, Burgess, Christopher, Glorot, Xavier,
  Botvinick, Matthew, Mohamed, Shakir, and Lerchner, Alexander.
\newblock beta-vae: Learning basic visual concepts with a constrained
  variational framework.
\newblock 2016.

\bibitem[Holding \& Lestas(2014)Holding and Lestas]{holding2014convergence}
Holding, Thomas and Lestas, Ioannis.
\newblock On the convergence to saddle points of concave-convex functions, the
  gradient method and emergence of oscillations.
\newblock In \emph{Decision and Control (CDC), 2014 IEEE 53rd Annual Conference
  on}, pp.\  1143--1148. IEEE, 2014.

\bibitem[Kim et~al.(2017)Kim, Cha, Kim, Lee, and Kim]{discogan2017}
Kim, Taeksoo, Cha, Moonsu, Kim, Hyunsoo, Lee, Jung~Kwon, and Kim, Jiwon.
\newblock Learning to discover cross-domain relations with generative
  adversarial networks.
\newblock \emph{CoRR}, abs/1703.05192, 2017.
\newblock URL \url{http://arxiv.org/abs/1703.05192}.

\bibitem[{Kingma} \& {Welling}(2013){Kingma} and
  {Welling}]{autoencoding_variational_bayes2013}
{Kingma}, D.~P and {Welling}, M.
\newblock {Auto-Encoding Variational Bayes}.
\newblock \emph{ArXiv e-prints}, December 2013.

\bibitem[Kingma \& Ba(2014)Kingma and Ba]{adam_optimization2014}
Kingma, Diederik and Ba, Jimmy.
\newblock Adam: A method for stochastic optimization.
\newblock \emph{arXiv preprint arXiv:1412.6980}, 2014.

\bibitem[Kodali et~al.(2018)Kodali, Hays, Abernethy, and
  Kira]{kodali2018convergence}
Kodali, Naveen, Hays, James, Abernethy, Jacob, and Kira, Zsolt.
\newblock On convergence and stability of gans.
\newblock 2018.

\bibitem[Kuleshov \& Ermon(2017)Kuleshov and Ermon]{kuleshov2017deep}
Kuleshov, Volodymyr and Ermon, Stefano.
\newblock Deep hybrid models: Bridging discriminative and generative
  approaches.
\newblock In \emph{Proceedings of the Conference on Uncertainty in AI (UAI)},
  2017.

\bibitem[Li et~al.(2017{\natexlab{a}})Li, Liu, Chen, Pu, Chen, Henao, and
  Carin]{alice2017}
Li, Chunyuan, Liu, Hao, Chen, Changyou, Pu, Yunchen, Chen, Liqun, Henao,
  Ricardo, and Carin, Lawrence.
\newblock Towards understanding adversarial learning for joint distribution
  matching.
\newblock 2017{\natexlab{a}}.

\bibitem[Li et~al.(2017{\natexlab{b}})Li, Song, and Ermon]{infogail2017}
Li, Yunzhu, Song, Jiaming, and Ermon, Stefano.
\newblock Inferring the latent structure of human decision-making from raw
  visual inputs.
\newblock 2017{\natexlab{b}}.

\bibitem[Liese \& Vajda(2006)Liese and Vajda]{liese2006divergences}
Liese, Friedrich and Vajda, Igor.
\newblock On divergences and informations in statistics and information theory.
\newblock \emph{IEEE Transactions on Information Theory}, 52\penalty0
  (10):\penalty0 4394--4412, 2006.

\bibitem[Liu \& Wang(2016)Liu and Wang]{stein_variational2016}
Liu, Qiang and Wang, Dilin.
\newblock Stein variational gradient descent: A general purpose bayesian
  inference algorithm.
\newblock \emph{arXiv preprint arXiv:1608.04471}, 2016.

\bibitem[Makhzani et~al.(2015)Makhzani, Shlens, Jaitly, and
  Goodfellow]{adversarial_autoencoder2015}
Makhzani, Alireza, Shlens, Jonathon, Jaitly, Navdeep, and Goodfellow, Ian.
\newblock Adversarial autoencoders.
\newblock \emph{arXiv preprint arXiv:1511.05644}, 2015.

\bibitem[Mescheder et~al.(2017)Mescheder, Nowozin, and
  Geiger]{adversarial_variational_bayes2017}
Mescheder, Lars, Nowozin, Sebastian, and Geiger, Andreas.
\newblock Adversarial variational bayes: Unifying variational autoencoders and
  generative adversarial networks.
\newblock \emph{arXiv preprint arXiv:1701.04722}, 2017.

\bibitem[Mohamed \& Lakshminarayanan(2016)Mohamed and
  Lakshminarayanan]{mohamed2016learning}
Mohamed, Shakir and Lakshminarayanan, Balaji.
\newblock Learning in implicit generative models.
\newblock \emph{arXiv preprint arXiv:1610.03483}, 2016.

\bibitem[Nowozin et~al.(2016)Nowozin, Cseke, and Tomioka]{f_gan2016}
Nowozin, Sebastian, Cseke, Botond, and Tomioka, Ryota.
\newblock f-gan: Training generative neural samplers using variational
  divergence minimization.
\newblock In \emph{Advances in Neural Information Processing Systems}, pp.\
  271--279, 2016.

\bibitem[Pu et~al.(2017)Pu, Wang, Henao, Chen, Gan, Li, and
  Carin]{pu2017adversarial}
Pu, Yuchen, Wang, Weiyao, Henao, Ricardo, Chen, Liqun, Gan, Zhe, Li, Chunyuan,
  and Carin, Lawrence.
\newblock Adversarial symmetric variational autoencoder.
\newblock In \emph{Advances in Neural Information Processing Systems}, pp.\
  4333--4342, 2017.

\bibitem[Radford et~al.(2015)Radford, Metz, and
  Chintala]{deconvolutional_gan2015}
Radford, Alec, Metz, Luke, and Chintala, Soumith.
\newblock Unsupervised representation learning with deep convolutional
  generative adversarial networks.
\newblock \emph{arXiv preprint arXiv:1511.06434}, 2015.

\bibitem[Ramdas et~al.(2015)Ramdas, Reddi, P{\'o}czos, Singh, and
  Wasserman]{ramdas2015decreasing}
Ramdas, Aaditya, Reddi, Sashank~Jakkam, P{\'o}czos, Barnab{\'a}s, Singh, Aarti,
  and Wasserman, Larry~A.
\newblock On the decreasing power of kernel and distance based nonparametric
  hypothesis tests in high dimensions.
\newblock In \emph{AAAI}, pp.\  3571--3577, 2015.

\bibitem[{Rezende} et~al.(2014){Rezende}, {Mohamed}, and
  {Wierstra}]{variational_dbn_stochastic_bp2014}
{Rezende}, D., {Mohamed}, S., and {Wierstra}, D.
\newblock {Stochastic Backpropagation and Approximate Inference in Deep
  Generative Models}.
\newblock \emph{ArXiv e-prints}, January 2014.

\bibitem[Shamir et~al.(2010)Shamir, Sabato, and Tishby]{shamir2010learning}
Shamir, Ohad, Sabato, Sivan, and Tishby, Naftali.
\newblock Learning and generalization with the information bottleneck.
\newblock \emph{Theoretical Computer Science}, 411\penalty0 (29-30):\penalty0
  2696--2711, 2010.

\bibitem[Strang et~al.(1993)Strang, Strang, Strang, and
  Strang]{strang1993introduction}
Strang, Gilbert, Strang, Gilbert, Strang, Gilbert, and Strang, Gilbert.
\newblock \emph{Introduction to linear algebra}, volume~3.
\newblock Wellesley-Cambridge Press Wellesley, MA, 1993.

\bibitem[Tishby \& Zaslavsky(2015)Tishby and
  Zaslavsky]{information_bottleneck2015}
Tishby, Naftali and Zaslavsky, Noga.
\newblock Deep learning and the information bottleneck principle.
\newblock \emph{CoRR}, abs/1503.02406, 2015.
\newblock URL \url{http://arxiv.org/abs/1503.02406}.

\bibitem[Tolstikhin et~al.(2017)Tolstikhin, Bousquet, Gelly, and
  Schoelkopf]{tolstikhin2017wasserstein}
Tolstikhin, Ilya, Bousquet, Olivier, Gelly, Sylvain, and Schoelkopf, Bernhard.
\newblock Wasserstein auto-encoders.
\newblock \emph{arXiv preprint arXiv:1711.01558}, 2017.

\bibitem[Yang et~al.(2017)Yang, Hu, Salakhutdinov, and
  Berg{-}Kirkpatrick]{improved_vae_nlp2017}
Yang, Zichao, Hu, Zhiting, Salakhutdinov, Ruslan, and Berg{-}Kirkpatrick,
  Taylor.
\newblock Improved variational autoencoders for text modeling using dilated
  convolutions.
\newblock \emph{CoRR}, abs/1702.08139, 2017.
\newblock URL \url{http://arxiv.org/abs/1702.08139}.

\bibitem[Zhao et~al.(2017)Zhao, Song, and Ermon]{zhao2017infovae}
Zhao, Shengjia, Song, Jiaming, and Ermon, Stefano.
\newblock Infovae: Information maximizing variational autoencoders.
\newblock \emph{CoRR}, abs/1706.02262, 2017.
\newblock URL \url{http://arxiv.org/abs/1706.02262}.

\bibitem[Zhu et~al.(2017)Zhu, Park, Isola, and Efros]{cyclegan2017}
Zhu, Jun{-}Yan, Park, Taesung, Isola, Phillip, and Efros, Alexei~A.
\newblock Unpaired image-to-image translation using cycle-consistent
  adversarial networks.
\newblock \emph{CoRR}, abs/1703.10593, 2017.
\newblock URL \url{http://arxiv.org/abs/1703.10593}.

\end{thebibliography}
